\documentclass{article}

\usepackage{PRIMEarxiv}

\usepackage[utf8]{inputenc} 
\usepackage[T1]{fontenc}    
\usepackage{hyperref}       
\usepackage{url}            
\usepackage{booktabs}       
\usepackage{amsfonts}       
\usepackage{nicefrac}       
\usepackage{microtype}      
\usepackage{lipsum}
\usepackage{fancyhdr}       
\usepackage{graphicx}       
\usepackage{amsthm}
\usepackage{times}
\usepackage{amsmath}
\usepackage{bbm}
\newtheorem{theorem}{Theorem}
\newtheorem{lemma}{Lemma}
\newtheorem{remark}{Remark}
\usepackage{natbib}
\allowdisplaybreaks
\title{Semiparametric Efficient Inference \\in Adaptive Experiments}

\author{Thomas Cook\thanks{Some of this work was performed while at J.P. Morgan Chase \& Co.} 
 \\ 
        Dept. of Mathematics and Statistics \\
        University of Massachusetts \\
        \texttt{tjcook@umass.edu} 
        \And
        Alan Mishler \\
        J.P. Morgan AI Research \\
        J.P. Morgan Chase \& Co. \\
        \texttt{alan.mishler@jpmchase.com}
        \And
        Aaditya Ramdas \\
        Dept. of Statistics \& Data Science and Machine Learning\\
        Carnegie Mellon University \\
        \texttt{aramdas@cmu.edu}
}

\begin{document}
\maketitle

\begin{abstract}
  We consider the problem of efficient inference of the \emph{Average Treatment Effect} in a sequential experiment where the policy governing the assignment of subjects to treatment or control can change over time. We first provide a central limit theorem for the Adaptive Augmented Inverse-Probability Weighted estimator, which is semiparametric efficient, under weaker assumptions than those previously made in the literature. This central limit theorem enables efficient inference at fixed sample sizes. We then consider a sequential inference setting, deriving both asymptotic and nonasymptotic confidence sequences that are considerably tighter than previous methods. These \emph{anytime-valid} methods enable inference under data-dependent stopping times (sample sizes). Additionally, we use propensity score truncation techniques from the recent off-policy estimation literature to reduce the finite sample variance of our estimator without affecting the asymptotic variance. Empirical results demonstrate that our methods yield narrower confidence sequences than those previously developed in the literature while maintaining time-uniform error control.

\end{abstract}


\section{Introduction}\label{sec:intro}

Randomized experiments with two treatment arms, also known as A/B tests, are widely used across many domains. Classical statistical tools (\emph{fixed-time} methods) require the analyst to select the sample size in advance and only perform inference when this sample size is reached. However, modern A/B testing platforms enable continuous monitoring of results, which allows analysts to make repeated decisions about whether to stop or continue an experiment based on the data observed so far. For example, an analyst might decide to run an experiment precisely until a test statistic becomes statistically significant, at which point they may stop and declare a treatment effective. When the test statistic is based on a fixed-time method, this can lead to inflated false positive (type-I error) rates. In fact, when used in this fashion, fixed-time methods based on the central limit theorem will in general cause these type-I error rates to go to 1 as $t \rightarrow \infty$, a result that is implied by the law of the iterated logarithm \citep{robbins1952sequential, johari2017PeekingTestsWhy}.

Statistical tools which enable valid inference in this setting are known as \emph{anytime-valid} methods. To illustrate the distinction, consider a confidence interval (CI) for a parameter of interest $\theta$. A $(1 - \alpha)$ CI for $\theta$ is an interval $[L_t, U_t]$ based on a sample of size $t$ with the property that
\begin{align}
    \forall t \in \mathbb{N}^+, \mathbb{P}(\theta \in [L_t, U_t]) \geq 1 - \alpha. \label{eq:ci_definition}
\end{align}
The coverage guarantee in \eqref{eq:ci_definition} only holds when the sample size (aka stopping time) $t$ is fixed in advance. By contrast, a \emph{confidence sequence} (CS) for $\theta$ is a sequence of intervals such that
\begin{align}
    \mathbb{P}(\forall t \in \mathbb{N}^+, \theta \in [L_t, U_t]) \geq 1-\alpha. \label{eq:cs_definition}
\end{align}
The coverage guarantee in \eqref{eq:cs_definition} is uniform in the sample size, which enables valid inference under data-dependent stopping times. This means that the analyst can continually monitor the experiment and adaptively choose when to stop without inflating the type-I error rate. A CS can be constructed so that time-uniform coverage is guaranteed asymptotically, a notion which is defined rigorously in \cite{waudbysmith2023timeuniform}.

Continuous monitoring also enables analysts to adaptively update the policy governing the assignment of subjects to treatment vs. control. Adaptive experiments in general enable more efficient estimation of treatment effects than non-adaptive experiments, such as a traditional experiment which randomly assigns each subject to treatment or control with probability 0.5 \citep{hahn2011adaptive}. Adaptive designs can therefore enable users to spend less resources running experiments while also minimizing the number of subjects that are exposed to ineffective or possibly harmful treatments. Additionally, mid-stream changes to experimental designs are sometimes imposed by considerations other than statistical efficiency, such as changes to budgets or unexpected impacts of the treatment on a business metric. It is therefore desirable to be able to perform inference under a wide range of adaptive settings.

As an example of this problem setting, consider a pharmaceutical company running a trial to test whether a treatment is effective or not to gain regulatory approval. The pharmaceutical company would like to gain approval quickly by concluding the trial using as few samples as possible. Without prior knowledge of the effect size, choosing an appropriate sample size is difficult. If an overly large sample size is chosen, then the company will have wastefully run the trial longer than necessary, keeping an effective treatment away from patients. However, if the chosen sample size is too small, then the trial may be inconclusive, in which case the entire cost of the trial is wasted. With standard fixed-time inference tools, the company must start the trial from the beginning, or abandon the treatment. If an anytime-valid method is used, then the company can simply continue the trial without worrying about inflating the false positive rate. To protect participants from potentially harmful treatments, it is common to track the results of an experiment as they are observed. It is also possible that a treatment is so overwhelmingly effective that the company would like to conclude the trial immediately and gain regulatory approval. All these settings involve data-dependent stopping times, which require anytime-valid methods rather than fixed-time methods.

In this paper, we consider inference for the Average Treatment Effect (ATE), which is the expected difference in outcomes between the two treatment arms, in the context of adaptive experiments. \citet{kato2021} proposed the Adaptive Augmented IPW (A2IPW) estimator, which, when coupled with a particular adaptive design, yields asymptotically efficient CIs based on the central limit theorem \citep{hahn2011adaptive}. Furthermore, \citet{kato2021} analyzed finite-sample regret (in terms of mean squared error) and showed that under certain conditions their adaptive design improves the regret bound compared to a non-adaptive design. They also provided a CS for the ATE using concentration inequalities based on nonasymptotic variants of the law of the iterated logarithm (LIL). 

In a similar, but independent line of research, \citet{dai2023clipogd} proposed an experimental design such that the variance of an adaptive IPW estimator asymptotically achieves the variance under the optimal Neyman design, the fixed (non-adaptive) design that minimizes the variance of the IPW estimator but is unknown in practice due to dependence on unknown parameters. They provided a treatment assignment policy that achieves sublinear regret, in terms of the estimator's variance, through the use of a variant of online stochastic projected gradient descent. They also provided an asymptotically-valid Chebyshev-type CI for the ATE. In contrast to our work, their work studied this problem from a design-based potential outcomes framework, which assumes deterministic potential outcomes, and did not consider contexts or covariates. Our work assumes that contexts and potential outcomes are drawn at random. We provide an asymptotically valid Wald-type CI for an A2IPW estimator and provide anytime-valid inference tools.

Our contributions are both theoretical and empirical. Theoretically, we provide semiparametric efficient (approximate) inference at fixed times under weaker (more general) conditions than in the previous literature. Empirically, we adapt state-of-the-art CSs to the setting of adaptive A/B tests and show superior performance in a sequential testing setting using a treatment assignment policy from our fixed-time theory. Explicitly, our contributions are:
\begin{itemize}
    \item \textbf{Theory: } We prove a central limit theorem for the A2IPW estimator under weaker assumptions than those utilized by \citet{kato2021}, enabling approximately valid inference at fixed sample sizes. While these results are valid for arbitrary adaptive designs (with some mild restrictions), we propose a design which adaptively truncates the treatment assignment probabilities for finite sample stability \citep{waudbysmith2022anytimevalid}. We show that this estimator is semiparametric efficient when paired with the proposed design.

    \item \textbf{Empirical Results: }We couple the A2IPW estimator with anytime-valid methods which yield much tighter intervals (more powerful inference) than the methods in \citet{kato2021}. We derive CSs based on test (super)martingales \citep{waudbysmith2022estimating}, as well as asymptotic CSs \citep{waudbysmith2023timeuniform}. 
\end{itemize}
To offer a demonstrative example, Figure~\ref{fig:single_iter} compares our CSs to previous work for a single sequential experiment. Our inference methods provide narrower intervals, leading to higher statistical power.

\begin{figure}
    \centering
    \includegraphics[width=0.85\textwidth]{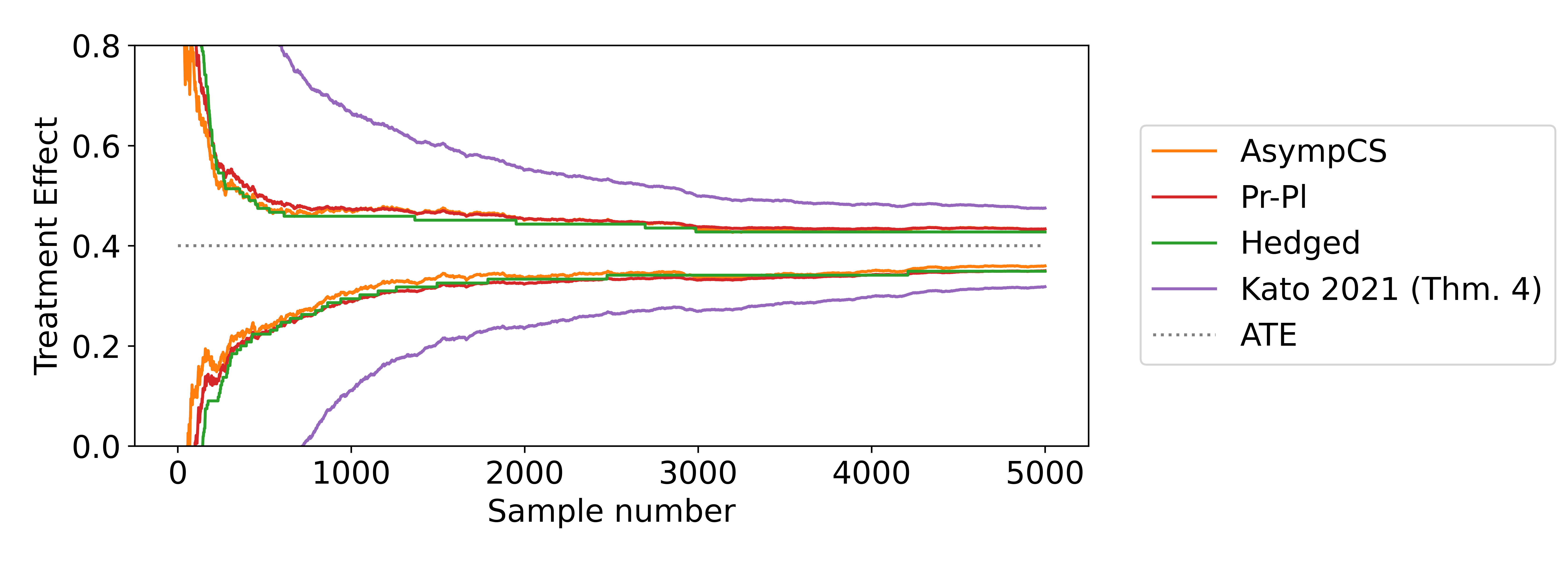}
    \caption{A single run of an experiment with bounded outcomes and the $ATE$ set to $0.4$ (simulation setup of Appendix~\ref{appdx:implementation_bounded} with $\pi_t \in [0.3,0.7]$). We propose confidence sequences (AsympCS, Pr-PI, Hedged) that are narrower than previous work \cite{kato2021}. }
    \label{fig:single_iter}
\end{figure}

\section{Problem Setting and Fixed-Time Inference}\label{sec:setting_and_fixedtime}

\subsection{Experimental Process}\label{sec:dgp}

We follow the same problem setting and data generating process as described in \cite{kato2021}, with minor modifications to their notation.
Subjects are indexed by $t \in \mathbb{N}^+$ and arrive sequentially. For each subject, the experimenter observes a context $X_t \in \mathcal{X}$, where $\mathcal{X}$ is the context domain, then assigns a treatment $A_t \in \{0, 1\}$, and then observes an outcome $Y_t \in \mathbb{R}$. 
We denote by $Y_t(a)$ the potential outcome corresponding to treatment $a$, for $a \in \{0, 1\}$, and we assume that $Y_t = \mathbbm{1}[A_t = 0]Y_t(0) + \mathbbm{1}[A_t = 1]Y_t(1)$, where $\mathbbm{1}[\cdot]$ denotes the indicator function. That is, we assume that a given subject's outcome depends only on their own treatment assignment and not on the treatment assignments of other subjects \citep{rubin1980sutva,Rubin1986CommentWI}.
The accumulated data after $T$ subjects (equivalently, $T$ time steps, where $T \in \{\mathbb{N}^+ \cup \infty\}$) consists of a set $\{(X_t, A_t, Y_t)\}_{t=1}^T$, whose distribution is 
\begin{equation*}
    (X_t, A_t, Y_t) \sim p(x)\pi_t(a \mid x,\Omega_{t-1}) p(y \mid a,x),
\end{equation*}
where $\Omega_{t-1} = \{(X_s, A_s, Y_s): s \leq t - 1\}$ denotes the \emph{history}. 
We denote the domain of $\Omega_{t-1}$ by $\mathcal{H}_{t-1}$.
We assume that $\{(X_t, Y_t(0), Y_t(1))\}_{t=1}^T$ are independent and identically distributed (iid).
However, our treatment assignments are not fixed over time, and depend on previous observations.
We define the propensity score $\pi_t(a \mid x, \Omega_{t-1})$ from the experimenter's \emph{policy}, $\pi_t : \mathcal{A} \times \mathcal{X} \times \mathcal{H}_{t-1} \mapsto [0, 1].$ 
Although the context and potential outcomes are independent over time, the observed outcomes $\{Y_t\}_{t=1}^{T}$ are dependent due to dependence in the policy. 

\begin{remark}
    Although we assume $\{X_t, Y_t(0), Y_t(1)\}_{t=1}^T$ to be iid, our results can be extended to a non-iid setting in which we estimate the running mean of the individual treatment effects, $\frac{1}{T} \sum_{t=1}^{T} \theta_t$, where $\theta_t = (Y_t(1) - Y_t(0))$ denotes the individual treatment effect for subject $t$. As a special case of that setting, when $\{X_t, Y_t(0), Y_t(1)\}_{t=1}^T$ are iid, we recover the ATE, since $\frac{1}{T} \sum_{t=1}^{T} \theta_t \to \theta_0$. 
\end{remark}

As data collection may be costly, time consuming, or high risk, the experimenter may not want to continue until some predetermined sample size. 
Conversely, an experimenter may reach a predetermined sample size and consider proceeding with further data collection, for example because they consider the results inconclusive. 
This type of data-dependent stopping requires methods which can handle peeking \citep{ramdas2023gametheoretic}, and is the focus of Section~\ref{sec:anytimevalid}.
Under the anytime-valid inference methods described in that section, the experimenter can choose to stop the experiment, continue under the current policy, or continue under a modified policy, without inflating the type-I error rate. 

\subsection{Estimating the Average Treatment Effect}

\paragraph{Additional notation:} Our notation follows \cite{kato2021} with minor modification. Let $a$ be a treatment in $\mathcal{A}$. We denote $\mathbb{E}[Y_t(a)\mid x]$, $\mathbb{E}[Y^2_t(a)\mid x]$, $\mathrm{Var}(Y_t(a)\mid x)$, and $\mathbb{E}[Y_t(1)  - Y_t(0)\mid x]$ as $f(a, x)$, $e(a, x)$, $v(a, x)$, and $\theta_0(x)$, respectively. Let $\hat{f}_{t}(a, x)$ and $\hat e_{t}(a, x)$ denote estimators of $f(a, x)$ and $e(a, x)$ constructed from $\Omega_{t}$, respectively.\footnote{In general, $\hat{f}$ can be any arbitrary estimator. In Theorem~\ref{theorem:MDS-CLT}, we simply require $\hat{f}$ to be consistent for $f$.} We denote the $\ell_2$ norm of a function as $\lVert f \rVert_2^2 = \int \left\{f(x)\right\}^2 d\mathbb{P}(x) $.

\paragraph{Adaptive Estimator:}
We denote the causal parameter of interest, the ATE, as $\theta_0 = \mathbb{E}(Y(1) - Y(0))$, where the subscript $t$ is dropped to emphasize time invariance.
In an experimental setting, the treatment probabilities are known and the Inverse-Probability Weighted (IPW) estimator produces an unbiased estimate of $\theta_0$. 
The Augmented IPW (AIPW) estimator extends the IPW estimator to include regression estimates, which can reduce the variance of the estimator, while maintaining unbiasedness \citep{robins1994aipw, chernozhukov2018ddml}.
Recently, \cite{kato2021} extended the AIPW estimator to the setting of an adaptive experiment by defining the \emph{Adaptive} AIPW estimator (A2IPW). 
The key difference between the two estimators is the use of data-dependent propensity scores $\pi_t$, which can be updated at each time point $t$ based on the accumulated data $\Omega_{t-1}$. 
The A2IPW estimator, given that $T$ subjects have been observed, is defined as $\hat{\theta}_T^{\mathrm{A2IPW}} = \frac{1}{T}\sum_{t=1}^T h_t$, where
\begin{equation*}\label{eq:ht}
h_t = \Bigg( \frac{\mathbbm{1}[A_t = 1] (Y_t - \hat{f}_{t-1}(1,X_t))}{\pi_t (1 \mid X_t, \Omega_{t-1})} - \frac{\mathbbm{1}[A_t = 0] (Y_t - \hat{f}_{t-1}(0,X_t))}{\pi_t (0 \mid X_t, \Omega_{t-1})} + {\hat{f}_{t-1}(1,X_t) - \hat{f}_{t-1}(0,X_t)} \Bigg) .
\end{equation*}

\cite{hadad2021adaptive} presented a more general form of an adaptive AIPW estimator. In their definition, the individual iterates are weighted by adaptive evaluation weights to guarantee asymptotic normality of the weighted average. $\hat{\theta}_T^{\mathrm{A2IPW}}$ can therefore be viewed as a special case in which the evaluation weights are set to be equal.

\cite{hahn2011adaptive} showed that the policy that minimizes the semiparametric lower bound of the asymptotic variance for regular estimators of the ATE is $\pi^{\mathrm{AIPW}}$, defined as 
$$\pi^{\mathrm{AIPW}}(1 \mid X_t) = \frac{\sqrt{v(1,X_t)}}{\sqrt{v(1,X_t)} + \sqrt{v(0,X_t)}}.$$
\cite{hahn2011adaptive} derived this lower bound in the context of a two-stage experimental design. \cite{armstrong2022asymptotic} showed that for estimating the ATE of a binary treatment, no further first-order asymptotic efficiency gain is possible in a purely adaptive experiment. This policy depends on unknown quantities of the underlying data generating process. 
The policy proposed in \cite{kato2021} estimates the unknown quantities from the observed data and is defined as 
\begin{equation*}\label{eq:pi_a2ipw_kato}
\pi_t^{\mathrm{A2IPW, Kato}}(1 \mid X_t, \Omega_{t-1}) = \frac{\sqrt{\hat{v}_{t-1}(1,X_t)}}{\sqrt{\hat{v}_{t-1}(1,X_t)} + \sqrt{\hat{v}_{t-1}(0,X_t)}} ,
\end{equation*}
where $\hat{v}_{t-1}$ denotes an estimate of $v$ using the first $t-1$ samples. For numerical stability, \cite{kato2021} mixed this policy with a non-adaptive policy that assigns treatment with probability $0.5$. As the sample size grows, the mixing gradually assigns a greater weight to the estimated optimal policy. This mixing scheme prevents noisy estimates of $v$ from inducing high variance in the observed $(h_t)_{t=1}^T$ early in the experiment and does not affect the asymptotic properties of the estimator. In a similar spirit, we explicitly define a truncation schedule for the propensity scores generated by our policy. However, our truncation schedule is not only useful for improving finite sample stability in practice; it is also a technical device that allows us to relax the assumptions needed for our results below.
We define our policy as
\begin{equation}\label{eq:pi_a2ipw}
    \pi_t^{\mathrm{A2IPW}}(1 \mid X_t, \Omega_{t-1}) = \left( \pi_t^{\mathrm{A2IPW, Kato}}(1 \mid X_t, \Omega_{t-1}) \vee \frac{1}{k_t} \right) \wedge \left(1 - \frac{1}{k_t}\right),
\end{equation}
where $k_t \in [2,\infty)$ is a user-chosen
truncation parameter. Since our theorem below holds in a more general setting, we can apply this truncation to arbitrary policies $\tilde{\pi}_t$, denoting 
\begin{equation}\label{eq:pi_arbitrary}
    \pi_t = \left( \tilde{\pi}_t  \vee \frac{1}{k_t} \right) \wedge \left(1 - \frac{1}{k_t}\right).
\end{equation}
Note that setting $k_t=2$ results in $\pi_t(A_t = 0 \mid X_t, \Omega_{t-1}) = \pi_t(A_t = 1 \mid X_t, \Omega_{t-1}) = 0.5$, and $k_t \to \infty$ results in the non-truncated policy $\tilde{\pi}_t$. 
The policy truncation that we utilize is inspired by \cite{waudbysmith2022anytimevalid}, where truncation circumvents required knowledge of the maximal importance weight in off-policy evaluation. Our empirical results show that truncation can improve finite-sample performance for well-chosen $k_t$, which aligns with the results in the preceding work.

\subsection{Fixed-Time Confidence Intervals}\label{sec:fixed_time_theory}
 
We now turn to constructing CIs with asymptotic coverage guarantees.
\cite{kato2021} defined $z_t = h_t - \theta_0$ and showed that $\{z_t\}_{t=1}^{T}$ forms a \emph{martingale difference sequence} (MDS). They then utilized a MDS central limit theorem to show $\hat{\theta}^{\mathrm{A2IPW}}$ is asymptotically Gaussian. They further showed that $\hat{\theta}^{\mathrm{A2IPW}}$ is semiparametric efficient under the asymptotic policy.
We provide the same results under weaker assumptions, as elaborated after the theorem.

\begin{theorem}[Asymptotic Distribution of $\hat{\theta}_T^{\mathrm{A2IPW}}$]\label{theorem:MDS-CLT} Assume $\{(X_t, A_t, Y_t ) \}_{t=1}^{T}$ follow the data generating process described in Section~\ref{sec:dgp}. Let $\tilde{\pi}_t: \mathcal{A} \times \mathcal{X} \mapsto (0, 1)$ be an arbitrary sequence of policies, and let $\pi_t$ be the corresponding truncated policies as defined in~\eqref{eq:pi_arbitrary}. Assume $\frac{1}{\pi(a \mid x)} < C_1 < \infty$ and $v(a,x) < C_2 < \infty$ for all $x \in \mathcal{X}$ and $a \in \{0,1\}$ for some constants $C_1$ and $C_2$. Assume $k_t \lVert \hat{f}_{t} - f \rVert_2 = o_{\mathbb{P}}(1)$ and $k_t \lVert \pi_t - \pi \rVert_2 = o_{\mathbb{P}}(1)$ for some policy $\pi$. Further, assume that  $\textrm{Var}(\hat{f}_t \mid \Omega_{t-1}) < C_3 <\infty$ and $\textrm{Var}(\pi_t \mid \Omega_{t-1}) < C_4 < \infty $ for all $x \in \mathcal{X}$, $a \in \{0,1\}$, and $t \in \{1,2,\dots\}$, for some constants $C_3$ and $C_4$. Under these assumptions, we have
    $$\sqrt{T} (\hat{\theta}_T^{\mathrm{A2IPW}} - \theta_0) \xrightarrow[]{d} N(0,\sigma^2),$$
where $\sigma^2$ is the semiparametric lower bound of the asymptotic variance for regular estimators of $\theta_0$ under the policy $\pi(a \mid x)$, given by $$\sigma^2 = \mathbb{E} \left[ \sum_{a=0}^1 \frac{v(a, X_t)}{\pi(a \mid X_t)} + \left( f(1,X_t) - f(0,X_t) - \theta_0 \right)^2 \right].$$
In particular, if we have $\pi = \pi^{\mathrm{AIPW}}$, then $\hat \theta_T^{\mathrm{A2IPW}}$ is semiparametric efficient.
\end{theorem}

Proof is provided in Appendix~\ref{appdx:mds_clt_proof}. Note that \cite{kato2021} assumed that $Y_t$ and $\hat{f}_t$ are uniformly bounded and that $\hat{f}_t$ and $\pi_t$ converge pointwise. By contrast, we only assume that $v(a,x)$ are uniformly bounded, that $\hat{f}_t$ and $\pi_t$ converge in $\ell_2$ norm, and that $\hat{f}_t$ and $\pi_t$ have uniformly bounded conditional variances. \cite{kato2021} also assumed that $\pi_t$ is uniformly bounded away from 0, whereas we only make this assumption on $\pi$, the stationary policy which $\pi_t$ converges to. By utilizing truncation, we can satisfy the assumption that the policies $\pi_t$ (at fixed times) are uniformly bounded. However, truncation is still relevant to relax convergence assumptions. Our proof uses a MDS central limit theorem given by \cite{dvoretzky1972}, which is used in a similar fashion by \cite{zhang2021}. 
This form of a MDS central limit theorem states conditions based on the second moment of $z_t$ \emph{conditional} on $\Omega_{t-1}$, and lends itself directly to weakening pointwise convergence assumptions on $\hat{f}$ and $\pi_t$. 

In contrast to \cite{kato2021}, truncation plays a key role in our derivation of the asymptotic distribution of $\hat{\theta}_T^{\mathrm{A2IPW}}$ and, in turn, the conditions required of $k_t$ are of particular interest. If $k_t$ increases to infinity, then as long as the non-truncated policy $\tilde \pi_t$ converges to some non-truncated policy $\tilde \pi$, the truncated policy $\pi_t$ will also converge to $\tilde \pi$ (and the theorem would apply as long as $k_t$ increased slowly enough that $k_t \max(\lVert \hat{f}_{t} - f \rVert_2, \lVert \pi_t - \pi \rVert_2) = o_{\mathbb{P}}(1)$). If, instead, $k_t$ remains constant or increases to a finite bound, then the truncated policy $\pi_t$ will converge to an appropriate truncation of $\tilde \pi$, and the theorem would still apply as long as $\max(\lVert \hat{f}_{t} - f \rVert_2, \lVert \pi_t - \pi \rVert_2) = o_{\mathbb{P}}(1)$. When we have $\pi = \pi^{\mathrm{AIPW}}$, we are implicitly assuming that we have independently selected $k_t$ such that truncation becomes asymptotically inactive. Note that if $\tilde{\pi}_t$ were uniformly bounded away from 0, as assumed in \cite{kato2021}, then we could simply set $k_t = 1/ \min(\tilde{\pi}_t(x),1- \tilde{\pi}_t(x))$ so that $\tilde{\pi}_t = \pi_t$, meaning we never actively truncate $\tilde{\pi}_t$. In that case, the conditions in \cite[Theorem 1]{kato2021} would imply the conditions in our theorem. The following remark alludes to how selecting $k_t$ can lead us to semiparametric efficient inference with our proposed policy $\pi_t^{\mathrm{A2IPW}}$.

\begin{remark}[Semiparametric Efficiency]\label{remark:semiparametric-efficient} 
Assume that we set $k_t$ such that $\lim_{t \to \infty} k_t > \sup \frac{1}{\pi^{\mathrm{AIPW}}}$. Assume that the estimated conditional variance function $\hat{v}_t$ is consistent for $v$ such that $\lVert \pi_t^{\mathrm{A2IPW}} - \pi^{\mathrm{AIPW}} \rVert_2 = o_{\mathbb{P}}(1)$. If $k_t$ grows at a rate such that
$ k_t \lVert \pi_t^{\mathrm{A2IPW}} - \pi^{\mathrm{AIPW}} \rVert_2 = o_{\mathbb{P}}(1) $ and all other assumptions of Theorem~\ref{theorem:MDS-CLT} hold, then $\hat{\theta}_T^{\mathrm{A2IPW}}$ is semiparametric efficient.
\end{remark}
 
In the final sentence of Theorem~\ref{theorem:MDS-CLT},  we state that if $\pi_t$ converges to $\pi^{\mathrm{AIPW}}$, then the semiparametric lower bound is minimized with respect to $\pi$ \citep{hahn2011adaptive}. In order to make use of this result, we require an adaptive policy that converges to $\pi^{\mathrm{AIPW}}$. Remark~\ref{remark:semiparametric-efficient} states that $\pi_t^{\mathrm{A2IPW}}$ is such a policy as long as our estimates of $v$ are consistent and our truncation does not vanish too quickly. The rate at which $k_t$ is allowed to increase as per the conditions in Remark~\ref{remark:semiparametric-efficient} depends on the rate that $\hat{v}_t$ converges to $v$. In practice this rate is unobservable, and it is worth acknowledging this limitation. Addressing this limitation is an interesting direction for future research.

A t-statistic, along with an explicit CI, are defined in Appendix~\ref{appdx:tstat}.
Although our interval is the same as the one given in \cite{kato2021}, our relaxed assumptions make its use applicable in more general settings, such as when the potential outcomes follow a distribution without bounded support.

\section{Anytime-Valid Inference in Adaptive Experiments}\label{sec:anytimevalid}

We now construct confidence sequences (CSs) for the ATE that utilize the A2IPW estimator. \cite{kato2021} developed such CSs via concentration inequalities based on the law of the iterated logarithm (LIL) as derived in \citet{balsubramani2015sharp} and \citet{balsubramani2016sequential}, which are today known to be loose in constants~\citep{Howard_2021}. The concentration inequality derived for $\hat{\theta}_T^{\mathrm{A2IPW}}$ \citep[Thm.\ 4]{kato2021} depends on the unknown treatment effect $\theta_0$, although we believe it is probably a trivial extension to replace this with an estimate of $\theta_0$. Indeed, though their derivation uses the true value $\theta_0$, their experiments use a running estimate for $\theta_0$ based on $\{h_t \}_{t=1}^{T-1}$.

In contrast, we present CSs for the ATE based on more recent, state-of-the-art methods for inference of means of random variables in sequential settings~\citep{waudbysmith2022estimating,waudbysmith2022anytimevalid}.
All our sequences are \emph{fully empirical}, meaning they do not depend on unknown parameters. We will see that these these methods empirically yield much tighter intervals than those previously derived. This section thus yields the first practically tight and theoretically sound CSs for using the semiparametric efficient A2IPW estimator in adaptive experiments.

\subsection{Betting Confidence Sequences}
We first derive a CS using results from \cite{waudbysmith2022estimating} and \cite{waudbysmith2022anytimevalid}. Since these CSs do not require independence between observations, their use in the setting of adaptive experimentation is natural. The approach is based on a set of \emph{capital processes}, each of which can be understood as the accumulated wealth of a gambler playing a game against nature. More precisely, we construct one capital process for each $\theta' \in \Theta$, the parameter space. At each time $t$, the confidence set corresponds to the set of $\theta' \in \Theta$ such that the respective capital process has not exceeded an improbable level of wealth for a fair game parameterized by $\theta'$. By constructing the capital process at $\theta_0$ to be a test martingale, the probability that the capital process \emph{ever} exceeds the value $\frac{1}{\alpha}$ is bounded by $\alpha$ \citep{Ville1939}. This gives time-uniform type-1 error control, allowing the analyst to reject any $\theta'$ for which the respective capital process exceeds $\frac{1}{\alpha}$. 

\begin{theorem}[{Hedged-CS [Hedged]}]\label{theorem:bettingCS}
    Assume we observe data following the data generating process of Section~\ref{sec:dgp}. Assume $Y_t \in [0,1]$ and $\pi_t(1 \mid X_t, \Omega_{t-1}) \in [\frac{1}{k_t}, 1-\frac{1}{k_t}]$ for all $t \in 1,\dots, T$. If we define
    \begin{equation*}
    \mathcal{K}_T^{+}(\theta') := \prod_{t=1}^{T}(1 + \lambda_t(\theta') (h_t - \theta')) ,
\hspace{4mm}
        \mathcal{K}_T^{-}(\theta') := \prod_{t=1}^{T}(1 - \lambda_t(\theta') (h_t - \theta')) ,
    \end{equation*}
    \begin{equation*}
        \mathcal{M}_T(\theta') := \frac{ \mathcal{K}_T^{+}(\theta') + \mathcal{K}_T^{-}(\theta')}{2},
    \end{equation*}
    then 
    \begin{equation*}
        C_T^{\text{Hedged}} := \bigcap_{t \leq T} \left\{ \theta' \in [-1,1] : \mathcal{M}_T(\theta') < \frac{1}{\alpha} \right\},
    \end{equation*}
    forms a $(1-\alpha)$-CS for $\theta_0$, where $(\lambda_t(\theta'))_{t=1}^T  \in \left(\frac{-1}{k_t - \theta'}, \frac{1}{k_t + \theta'} \right)$ is a predictable sequence that may be interpreted as an analyst's betting strategy.
\end{theorem}

Proof and other details can be found in Appendix~\ref{appdx:nonasymp_proof}. Before providing intuition for the proof, we note that although we assume $Y_t \in [0,1]$, our result holds for any bounded $Y_t$ by rescaling. 

$\mathcal{K}_T^{+}(\theta')$ and $\mathcal{K}_T^{-}(\theta')$ can be interpreted as capital processes for a gambler who is betting in favor of $\theta_0 > \theta'$ and $\theta_0 < \theta'$ respectively. Since we wish to produce two-sided intervals, we take the mean of these two process to form $\mathcal{M}_T(\theta')$. This is equivalent to a gambler partitioning their wealth equally between two games. The analyst must choose a predictable betting strategy for each game, $\lambda_t(\theta')$. In theory, this betting strategy could be different at each possible value of $\theta'$. Moreso, apart from a bounded range, the only restriction on $\lambda_t(\theta')$ is that it is predictable, meaning that it cannot depend on the current or any future observations. However, since our parameter space is continuous, an exhaustive search over an infinite set of $\theta'$ is not feasible. \cite{waudbysmith2022estimating} propose a method to set $\lambda_t$ to be quasi-convex in $\theta'$ so that the confidence set forms an interval. With quasi-convexity, it is sufficient to partition the parameter space and perform a grid-search; see their paper for further details and a variety of settings of $\lambda_t$. For brevity, we defer the details of setting $\lambda_t$ to Appendix~\ref{appdx:hedgedcs_subappendix}.

\subsection{Empirical Bernstein Confidence Sequences}
The confidence set produced by Theorem~\ref{theorem:bettingCS} can be computationally expensive, as a grid search is performed over $\theta' \in [-1,1]$.
A significant drawback is a lack of closed-form presentation. 
In this section we present a closed-form CS which has slight degradation in performance, but enjoys faster computation. 
This CS is based on an empirical Bernstein-type process that is shown to be a test supermartingale \citep{waudbysmith2022estimating}. Since this process inverts a test supermartingale, the concentration inequality is a looser bound than those produced by test martingales. 

Without loss of generality, assume that we observe $Y_t \in [0,1]$, for all $t \in 1, \dots, T$, and that the propensity scores, $\pi_t$, are all truncated to fall in $[ \frac{1}{k_t} , 1- \frac{1}{k_t}]$. Following a similar technique as in \cite{waudbysmith2022anytimevalid}, we define
\begin{equation}\label{eq:xi_def}
    \xi_t = \frac{h_t}{k_t + 1}, \hspace{4mm}
\hat{\xi}_{t-1} = \left(\frac{1}{t-1} \sum_{i=1}^{t-1} \xi_i \right) \wedge \frac{1}{k_t + 1},
\hspace{3mm}\textrm{ and}\hspace{3mm}
 \psi_E(\lambda) = - \log(1- \lambda) - \lambda.
\end{equation}
$\xi_t$ can be viewed as a scaled version of $h_t$. $\hat{\xi}_{t-1}$ is then a sample average of $\xi$ up through observation $t-1$. By only using previous observations, this value is \emph{predictable}, whereas the quantity $\bar{\xi}_t$, defined below in equation~\eqref{eq:sigmahat}, uses the current observation and is therefore not predictable. The scaling in $\xi_t$ and truncation in $\hat{\xi}_{t-1}$ are necessary technical tools to construct a test supermartingale, as shown by \citet{waudbysmith2022estimating}.

Similarly to the Hedged-CS of Theorem~\ref{theorem:bettingCS}, there are user-specified parameters, $(\lambda_t)_{t=1}^T$, which have an effect on the finite-sample performance of our forthcoming CS. $(\lambda_t)_{t=1}^{T}$ can be any $(0,1)$-valued predictable process. \cite{waudbysmith2022anytimevalid} provide an empirically promising setting, inspired by fixed-time empirical Bernstein CIs,
\begin{equation}\label{eq:lambda}
\lambda_t = \sqrt{\frac{2 \log (2/ \alpha)}{\hat{\sigma}^2_{t-1} t \log(1 + t)}} \wedge c, \textrm{ where }c = 0.5, 
\end{equation}
\begin{equation}\label{eq:sigmahat}
\hat{\sigma}^2_t = \frac{\sigma^2_0 + \sum_{i = 1}^t (\xi_i - \bar{\xi_i})^2}{t+1},\hspace{4mm}\textrm{ and}\hspace{4mm}
\bar{\xi}_t = \left(\frac{1}{t} \sum_{i = 1}^{t} \xi_i\right) \wedge \frac{1}{k_t + 1}.
\end{equation}

$\hat{\sigma}^2_t$ and $\bar{\xi}_t$ can be interpreted as estimates of the mean and variance of $\xi$. The value $\sigma^2_0$ can be viewed as a prior guess for the variance of $\xi$, and setting $\sigma^2_0 = \frac{1}{4}$ is a reasonable choice. We are now ready to present the CS.

\begin{theorem}[{Predictable Plug-In Empirical Bernstein CS [Pr-PI]}]\label{theorem:prpl_empbern}
    Assume we observe data following the data generating process of Section~\ref{sec:dgp}. Assume that $Y_t \in [0,1]$ and $\pi_t(1 \mid X_t, \Omega_{t-1}) \in [\frac{1}{k_t}, 1-\frac{1}{k_t}]$ for all $ t \in 1,\dots, T.$ Let $\xi_t$, $\hat{\xi}_{t-1}$, $\psi_E(\lambda)$, $\lambda_t$, $\hat{\sigma}^2_t,$ and $\bar{\xi}_t$ be defined as in \eqref{eq:xi_def}, 
    \eqref{eq:lambda}, and \eqref{eq:sigmahat} 
    respectively. We have that
    \begin{equation*}
        C_T^{Pr-PI} := \frac{\sum_{t=1}^T \lambda_t \xi_t}{\sum_{t=1}^T \lambda_t/(k_t + 1)} \pm \frac{\log(2/\alpha) + \sum_{t=1}^T \left(\xi_t - \hat{\xi}_{t-1} \right)^2 \psi_{E}(\lambda_t)}{\sum_{t=1}^T \lambda_t / (k_t + 1)} ,
    \end{equation*} 
    forms a $(1-\alpha)$ CS for $\theta_0$.
\end{theorem}

Proof is provided in Appendix~\ref{appdx:prpl_proof}, and follows the truncation technique used in \cite{waudbysmith2022anytimevalid}. When $k_t$ does not grow quickly, meaning propensities are truncated more aggressively, intervals tend to be narrower at small $t$. When $k_t$ grows quickly, then finite sample performance is sacrificed for performance at large $t$. The effect of truncation is consistent with the Hedged-CS.

\subsection{Asymptotic Confidence Sequences}

Because of their stronger time-uniform guarantees, the CSs in the preceding section produce intervals that have larger widths than their fixed-time CI counterparts. 
In the fixed-time setting, exact error coverage is not guaranteed, and the analyst must rely on asymptotic coverage guarantees. 
\cite{waudbysmith2023timeuniform} introduce the sequential analogue of asymptotic CIs, asymptotic CSs (AsympCS), by defining a  sequence of intervals which converges to some (unknown) CS. 
We now define our AsympCS for $\theta_0$.

\begin{theorem}[{Asymptotic CS [AsympCS]}]\label{theorem:asympCS}
Assume $\{(X_t, A_t, Y_t ) \}_{t=1}^{T} $ follow the data generating process described in section~\ref{sec:dgp}. Furthermore, assume $\mathbb{E}(|Y_t|^{2 + \delta}) < \infty$ for some $\delta > 0$. Let $\hat{\sigma}^2$ be an estimator of $\textrm{Var}(h_t)$, and $\rho > 0$ be a user-specified parameter. For all $ t \in 1,\dots, T$, we have that
$$C_T^{\textrm{AsympCS}} := \left( \frac{1}{T}\sum_{t = 1}^T h_t \pm \sqrt{ \frac{2(T \hat{\sigma}^2_T \rho^2 + 1)}{T^2 \rho^2} \log \left(\frac{\sqrt{T \hat{\sigma}^2_T \rho^2 + 1}}{\alpha} \right) }  \right),$$
forms a $(1-\alpha)$-AsympCS for $\theta_0$. Further, the width of $C_T^{\textrm{AsympCS}}$ at time T can be (approximately) minimized by setting $$\rho =\sqrt{\frac{-2 \log\alpha + \log (-2 \log \alpha + 1)}{T}}.$$
\end{theorem}

Proof is provided in Appendix~\ref{appdx:asympcs_proof}. The value of $\rho$ that minimizes interval width at a certain sample size follows from \citet[Appendix B.2]{waudbysmith2023timeuniform}. Although this interval does not yield exact coverage, empirically most errors occur quite early during the experiment. Its applicability for \emph{reasonable} sample sizes provides a noticeable gain in power in comparison to the exact CSs. 
The only user-specified parameter is $\rho$, a positive scalar which specifies at what intrinsic time the AsympCS should be tightest, with lower values corresponding to tightness at earlier times. We set this parameter to $0.5$ in our simulations; for more details, see \cite{waudbysmith2023timeuniform}.
We also note that the theorem we make use of from \cite{waudbysmith2023timeuniform} allows for time-varying conditional means. This suggests that the results of Theorem~\ref{theorem:asympCS} can be extended to time-varying treatment effects, which we leave for future work.

\section{Empirical Results}\label{sec:simulations}

We empirically compare our methods to \citet[Thm.\ 4]{kato2021}. We run two simulations with 1000 iterations each: one with Bernoulli outcomes, and one with continuous, bounded outcomes. $5000$ total samples are collected for each iteration and intervals are constructed following each sample. We employ sequential sample-splitting on $\hat{f}$ and $\hat{e}$ to avoid double-dipping and overfitting \citep{waudbysmith2023timeuniform}. Sequential sample-splitting permanently allocates each sample to one of two data folds upon observation. We fit models for $\hat{f}$ and $\hat{e}$ separately on each fold, giving four models in total. Predictions of $\hat{f}$ and $\hat{e}$ are produced from the model fit from the opposite fold. For an individual observation, we estimate the conditional variance by setting $\hat{v}(a,x) = \hat{e}(a,x) - (\hat{f}(a,x))^2$.  When determining $\pi^{\mathrm{A2IPW}}_t$, $\hat{f}$ and $\hat{e}$ are calculated by averaging predictions of the models from both splits, as this calculation occurs prior to observing and assigning the data point to a split. In our simulations, we clip $\hat{v}$ to be no less than 0.01 to avoid division by zero or negative values. During the first 100 samples, $\hat{f}(1,X_t) = 1$, $\hat{f}(0,X_t) = 0$, and $\pi_t = 0.5$. For policy truncation, we set $k_t = \frac{k_{t-1}}{0.999}$ where $k_1 = 2$. Since the method of \cite{kato2021} does not utilize time-varying bounds (at least in its present form), using the \emph{worst-case} bound for the propensities is a conservative way to guarantee time-uniform validity of their CS. Using our proposed truncation scheme can then make these CSs extremely wide. To remedy this, we observe that setting $k_t = 5$ works well for \citet[Thm.\ 4]{kato2021}. Figure~\ref{fig:miscoverage-bernoulli} shows results for these simulations when a Random Forest is used for $\hat{f}$ and $\hat{e}$, and are discussed in detail in the subsequent subsections. Appendix~\ref{appdx:implementation} provides additional results using a k-Nearest Neighbors model. 

\begin{figure*}[t]
    \includegraphics[width = \textwidth]{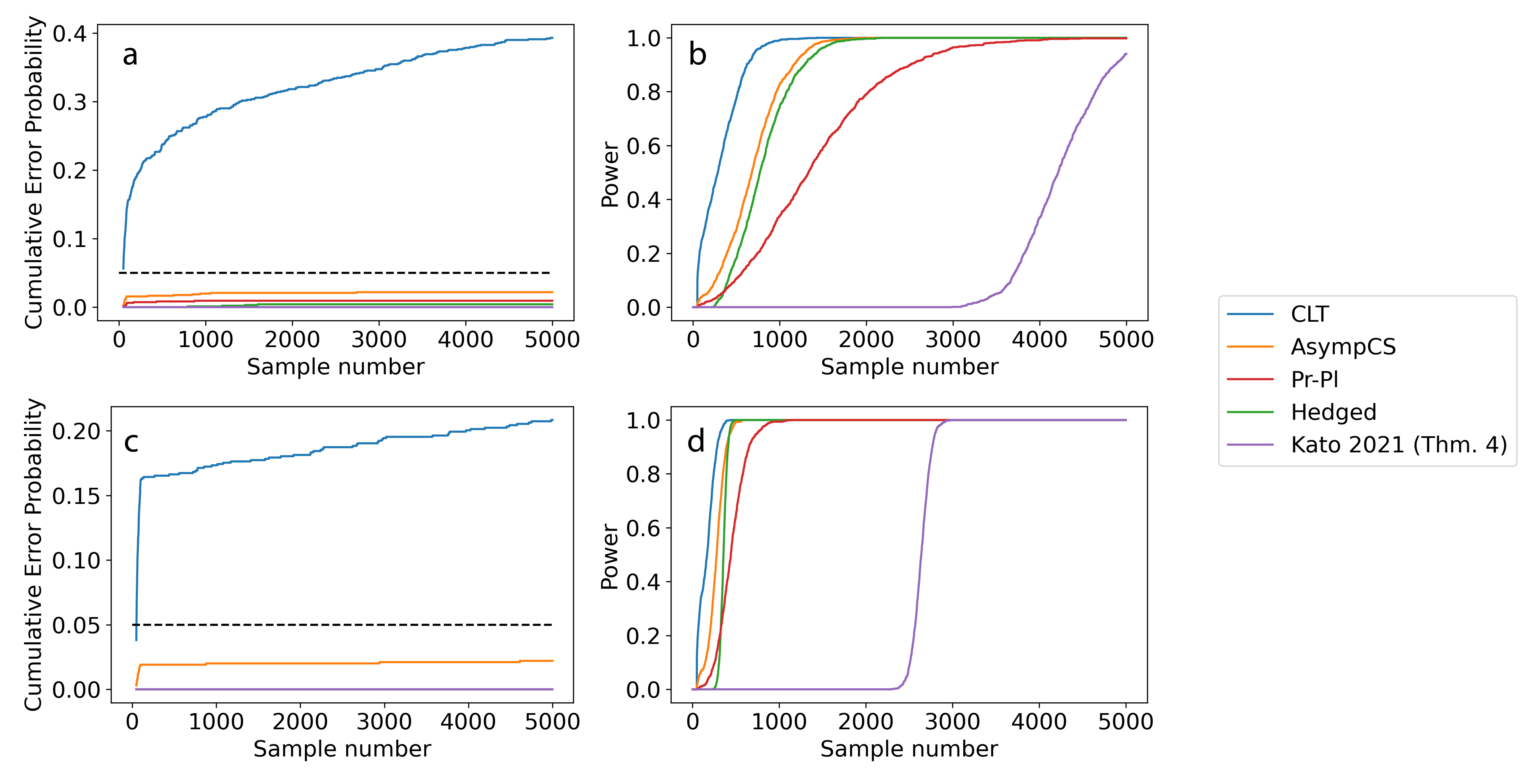}
    \caption{Cumulative error probability (a, c) and power (b, d) as functions of sample size, of experiments from Appendix~\ref{appdx:implementation_bernoulli} and Appendix~\ref{appdx:implementation_bounded}. The first row corresponds to the experiment with Bernoulli outcome, and the bottom row corresponds to the experiment with bounded, continuous outcomes. Intervals based on the CLT (Theorem~\ref{theorem:MDS-CLT}), AsympCS (Theorem~\ref{theorem:asympCS}), Pr-PI (Theorem~\ref{theorem:prpl_empbern}), Hedged (Theorem~\ref{theorem:bettingCS}), and ~\cite[Theorem 4]{kato2021} begin at $t = 50$. }
    \label{fig:miscoverage-bernoulli}
\end{figure*}

\subsection{Bernoulli Outcomes}\label{sec:bernoulli-sim}

In Figure~\ref{fig:miscoverage-bernoulli}, plots (a) and (b) show aggregated results across $1000$ iterations of a simulation with Bernoulli outcomes where the $ATE=0.1$. Full details of the data generating process can be found in Appendix~\ref{appdx:implementation_bernoulli}. Our methods provide significantly narrower intervals due to leveraging tighter concentration inequalities, as well as using time-varying truncation. Performance between our proposed CSs is inline with expectations from the CS literature. Specifically, the AsympCS provides the narrowest interval, at the expense of higher miscoverage probabilities. Since we begin constructing intervals at $t=50$, we avoid experiencing severe interval miscoverage early in the experiment, which prevents inflating the cumulative miscoverage rate. By doing so, we have empirically controlled the time-uniform error probability at level $\alpha = 0.05$. The Hedged-CS provides tighter intervals than the closed form Pr-PI. This is because the Hedged-CS inverts a test martingale instead of a test supermartingale. \citet{waudbysmith2022estimating} note that this removes a source of conservatism in generating CSs; however, this comes at a higher computational cost.

\subsection{Bounded Outcomes}\label{sec:bounded-sim}

Plots (c) and (d) in Figure~\ref{fig:miscoverage-bernoulli} show aggregated results across $1000$ iterations of a simulation with a bounded, continuous outcomes where the $ATE=0.1$. The data generating process can be found in Appendix~\ref{appdx:implementation_bounded}. Results in this section follow an identical pattern to the results in the Bernoulli experiment.

The data generating process yields data which is noticeably heteroskedastic. This implies that $\pi^{\mathrm{AIPW}}$ will produce values that are close to $0$ and $1$. However, when $t$ is small, $\hat{f}$ and $\hat{e}$ can be noisy estimates of the true, unobservable $f$ and $e$. In this case, the truncation schedule plays an interesting balance of preventing noisy estimates of $f$ and $e$ from inducing high variance in $\bar{h}_t$, while still allowing $\pi_t$ to converge to the optimal policy at an appreciable rate.

\subsection{Effect of Truncation Schemes}\label{sec:truncation-sim}

\begin{figure}
\includegraphics[width= \textwidth]{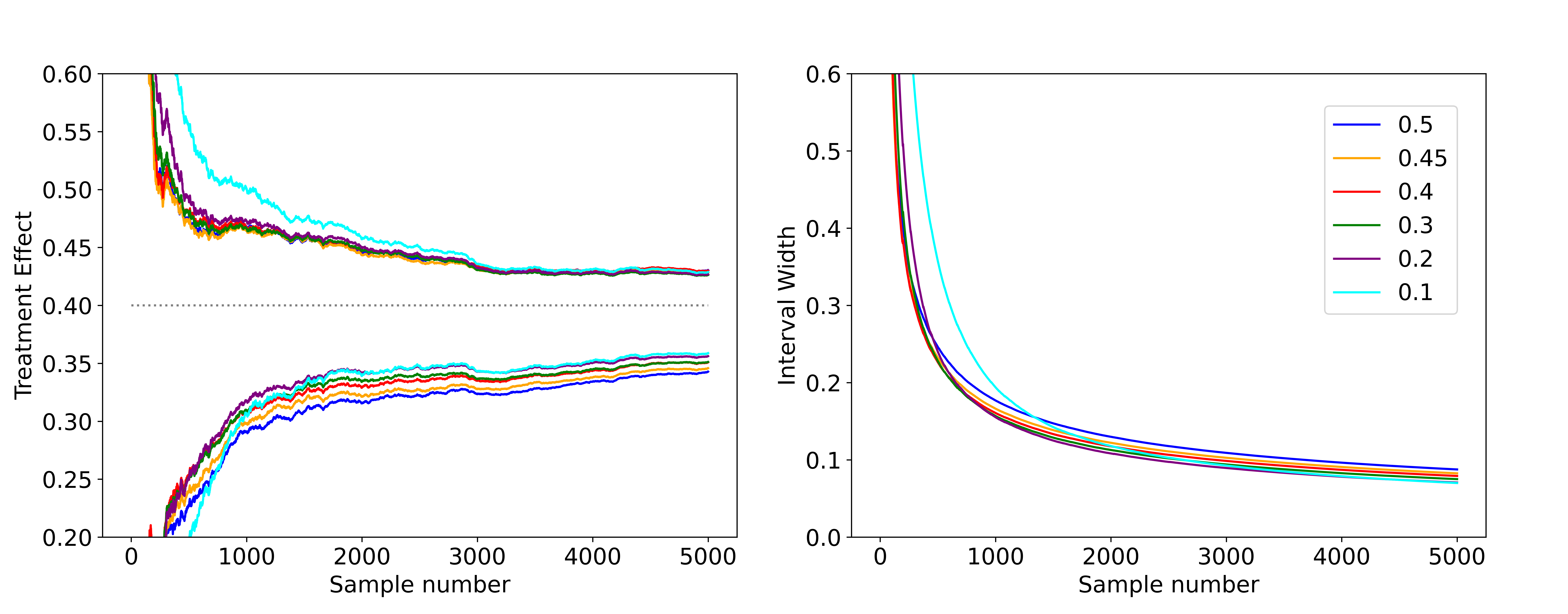}
    \caption{When $\pi_t$ is bounded in a narrower range, intervals produced by a Pr-PI CS are narrower at smaller $t$. }
    \label{fig:truncate}
\end{figure}

The policy studied in this work is deemed optimal because it minimizes the asymptotic variance of an unbiased estimator, the A2IPW estimator. The width of a CI based on the CLT has a direct dependence on the asymptotic variance of the estimator. Naturally, minimizing the asymptotic variance leads to a sense of optimal inference, by minimizing the mean squared error (MSE). In our proof of Theorem~\ref{theorem:MDS-CLT}, we make use of $k_t$ to bound propensity scores away from $0$ and $1$. In turn, we require that $k_t \lVert \hat{f}_t - f \rVert_2 = o_\mathbb{P}(1)$ and $k_t \lVert \pi_t - \pi \rVert_2 = o_\mathbb{P}(1)$. The rate at which $k_t$ increases is limited by the rates that $\lVert \hat{f}_t - f \rVert_2 \xrightarrow[]{p} 0$ and $\lVert \pi_t - \pi \rVert_2 \xrightarrow[]{p} 0$. We also use truncation as a technical tool when considering bounds on $(\lambda_t)_{t=1}^T$ in Theorem~\ref{theorem:bettingCS} and Theorem~\ref{theorem:prpl_empbern}.

Our primary concern in this work lies in anytime-valid inference, and as such, greater attention towards the width of the intervals produced by our CSs at fixed times is warranted. Since the propensity scores set by our policy appear in the denominator of $\hat{\theta}_T^{\text{A2IPW}}$, propensity scores near $0$ or $1$ can make $h_t$ arbitrarily large. The CSs with fixed-time error control considered in this paper make use of the boundedness of $h_t$. Particularly, the proofs make use of an underlying test (super)martingale, which by construction is non-negative. For example, non-negativity is guaranteed by scaling $\lambda_t(\theta')$ such that $\lambda_t(\theta')(h_t - \theta') > -1$ for the Hedged-CS. Temporarily subscribing to the betting analogy of \cite{waudbysmith2022estimating}, an inherent trade-off arises where the analyst must balance the allowable size of their bet, $\lambda_t(\theta')$, with the bounds of the evidence presented by nature, $(h_t - \theta')$. The opportunity to observe large evidence comes at the cost of placing small bets. 

This effect is noted explicitly by \citet[Remark 2]{waudbysmith2022anytimevalid}. In our setting, their intuition implies that faster growth in $k_t$ will yield a smaller asymptotic variance at the cost of having wider intervals at small $t$. In this section, we empirically show that a departure from our optimal policy through truncation will yield narrower intervals at finite times.

We consider a simulation that follows a similar set up to that used in Section~\ref{sec:bounded-sim}, where we modify $k_t$ to be constant. Specifically we set $k_t = 1/\pi_{t,min}$ and we vary $\pi_{t,min}\in \{ 0.5, 0.45, 0.40, $ $0.30,0.20,0.10\}$. For an explicit data generating process, see Appendix~\ref{appdx:more_sim}. We note that $\pi^{\mathrm{AIPW}}$, can be close to $0$ or $1$, and as a result, we truncate the \emph{optimal} policy. Results of a single iteration are shown in Figure~\ref{fig:truncate}. More aggressive truncation (larger values of $\pi_{t,min}$) leads to narrower intervals for small $t$, however, once $t$ is sufficiently large, less aggressive truncation (smaller values of $\pi_{t,min}$) provides narrower intervals. These results suggest that optimizing an adaptive policy for statistical inference at finite times is an interesting direction for future work.

\section{Conclusion}\label{sec:conclusions}

We have provided both confidence intervals (CIs) and confidence sequences (CSs) for the ATE in adaptive experiments using the A2IPW estimator. The CI based on the CLT achieves the semiparametric lower bound of the asymptotic variance under weaker assumptions than in previous work. The CSs with time-uniform error control surpass the performance of previous work considerably. Our methods apply to arbitrary adaptive designs, but we also propose a particular policy truncation scheme that preserves the asymptotic efficiency of the A2IPW estimator while improving finite sample performance. We emphasize that the inference tools (the CIs and CSs) and our proposed policy are individual contributions which do not require the use of one another.

This work provides a clear framework for using the A2IPW estimator in adaptive experiments. There are many interesting directions for further research. As mentioned in Section~\ref{sec:truncation-sim}, the truncation scheme used can have a considerable impact on finite sample performance. In the context of bounded random variables, \cite{shekhar2023nearoptimality} derive lower bounds on the width of a CS and show that betting-based confidence sets are nearly optimal. In our adaptive experiment, the analyst sets the bounds of the observed random variable $h_t$, and could explore minimizing the lower bound of the width of our CSs using their results. Separately, our adaptive design focuses purely on efficient statistical inference. In bandit experiments, an analyst may wish to minimize the \emph{regret}, so that patient welfare is maximized by treatment assignments. An interesting direction for future work is incorporating a notion of regret in the treatment assignments, such as the scheme proposed by \cite{simchi2023regret}. Lastly, we have only considering performing inference on the ATE. An interesting line of future work would be extending these results to other causal estimands.

\paragraph{Acknowledgements}
AR acknowledges support from NSF grants IIS-2229881 and DMS-2310718. 
This paper was prepared for informational purposes in part by the Artificial Intelligence Research group of JPMorgan Chase \& Co and its affiliates (“J.P. Morgan”) and is not a product of the Research Department of J.P. Morgan. J.P. Morgan makes no representation and warranty whatsoever and disclaims all liability, for the completeness, accuracy or reliability of the information contained herein. This document is not intended as investment research or investment advice, or a recommendation, offer or solicitation for the purchase or sale of any security, financial instrument, financial product or service, or to be used in any way for evaluating the merits of participating in any transaction, and shall not constitute a solicitation under any jurisdiction or to any person, if such solicitation under such jurisdiction or to such person would be unlawful.

\bibliography{arxiv_main}

\appendix

\section{Proof of Theorem~\ref{theorem:MDS-CLT}}\label{appdx:mds_clt_proof}

\subsection{High-Level Roadmap}
Our proof follows a similar style to the proof of \cite{kato2021}. 
We consider a martingale difference sequence (MDS) and apply a central limit theorem to find the asymptotic distribution of the sample mean of the MDS. 
The main departure of our proof from their proof is the statement of the central limit theorem which is amenable to making assumptions standard in causal inference.

To outline the proof, first we state our assumptions. Next we establish that $\{z_t\}_{t=1}^T$, where $z_t = h_t - \theta_0$, is a MDS. We then state the MDS central limit theorem by \cite{dvoretzky1972} and show that $\{z_t\}_{t=1}^T$ satisfies the necessary conditions. For the sake of brevity, we defer much of the tedious algebra to Appendix~\ref{appdx:auxiliary_proofs}. Since $\bar{z}_T = T^{-1}\sum_{t=1}^T z_t = T^{-1}\sum_{t=1}^T (h_t- \theta_0) = \hat{\theta}_T^{\mathrm{A2IPW}} - \theta_0$, this result allows us to characterize the asymptotic distribution of $\hat{\theta}_T^{\mathrm{A2IPW}}$. 

\subsection{Assumptions}\label{appdx:assumptions}
\begin{itemize}
\item \textit{\underline{IID Contexts and Potential Outcomes} }:  $\{X_t, Y_t(0), Y_t(1) \}_{t=1}^T$ are independent and identically distributed. 
\item \textit{\underline{Convergence of Regression }}:
$k_t\lVert \hat{f}_{t} - f \rVert_2 = o_{\mathbb{P}}(1)$ 
\item \textit{\underline{Convergence of Policy }}: $k_t\lVert \pi_t - \pi \rVert_2 = o_{\mathbb{P}}(1)$. 
\item \textit{\underline{$\pi$ Bounded Away from 0 }}:  $\frac{1}{\pi} < C_1$ for all $x \in \mathcal{X}$, and $a \in \{0,1\}$, where $C_1$ is a constant such that $C_1 < \infty$. 
\item \textit{\underline{Finite Conditional Variance }}:  $v(a,x) < C_2$  for all $x \in \mathcal{X}$, and $a \in \{0,1\}$, where $C_2$ is a constant such that $C_2 < \infty$. 
\item \textit{\underline{Finite Conditional Variance of Predictions }}: $\text{Var}(\hat{f}_{t-1}(a, X_t) \mid \Omega_{t-1}) < C_3$ for some $C_3 <\infty$ for all $x \in \mathcal{X}$, $a \in \{0,1\}$, and $t \in \{1,2,\dots\}$.
\item \textit{\underline{Finite Conditional Variance of Policy }}: $\text{Var}(\pi_t \mid \Omega_{t-1}) < C_4 $ for some $C_4<\infty$ for all $x \in \mathcal{X}$, $a \in \{0,1\}$, and $t \in \{1,2,\dots\}$.
\end{itemize}

\subsection{$z_t$ is a MDS}
 \cite{kato2021} show the first necessary condition, $\mathbb{E}(z_t \mid \Omega_{t-1}) = 0$. For completeness we present this step here.

\begin{align*}
&\mathbb{E}\big[z_t\mid \Omega_{t-1}\big]\\
&= \mathbb{E}\left[\frac{\mathbbm{1}[A_t=1]\big(Y_t - \hat{f}_{t-1}(1, X_t)\big)}{\pi_t(1\mid X_t, \Omega_{t-1})} - \frac{\mathbbm{1}[A_t=k]\big(Y_t - \hat{f}_{t-1}(0, X_t)\big)}{\pi_t(0\mid X_t, \Omega_{t-1})} + \hat{f}_{t-1}(1, X_t) - \hat{f}_{t-1}(0, X_t) - \theta_0 ~\Bigg|~ \Omega_{t-1}\right]\\
&= \mathbb{E}\Bigg[\hat{f}_{t-1}(1, X_t) - \hat{f}_{t-1}(0, X_t) - \theta_0\\
&\ \ \ \ \ \ \ \ \ \ \ + \mathbb{E}\left[\frac{\mathbbm{1}[A_t=1]\big(Y_t - \hat{f}_{t-1}(1, X_t)\big)}{\pi_t(1\mid X_t, \Omega_{t-1})} - \frac{\mathbbm{1}[A_t=0]\big(Y_t - \hat{f}_{t-1}(0, X_t)\big)}{\pi_t(0\mid X_t, \Omega_{t-1})} ~\Big|~  X_t, \Omega_{t-1}\right]~\Bigg|~\Omega_{t-1}\Bigg]\\
&= \mathbb{E}\left[\hat{f}_{t-1}(1, X_t) - \hat{f}_{t-1}(0, X_t) - \theta_0 + f(1, X_t) - f(0, X_t) - \hat{f}_{t-1}(1, X_t) + \hat{f}_{t-1}(0, X_t) ~\Big|~ \Omega_{t-1}\right] = 0.
\end{align*}

 The second required condition is $\mathbb{E}|z_t|< \infty$. By assumption $\mathbb{E}(z_t^2 \mid \Omega_{t-1}) < M < \infty$, where $M$ is some constant. This follows by uniformly bounded variance assumptions, since $\mathbb{E}(z_t \mid \Omega_{t-1}) = 0$. Since $\mathbb{E}(z_t^2) = \mathbb{E}(\mathbb{E}(z_t^2 \mid \Omega_{t-1}))$, it follows that $\mathbb{E}(z_t^2) < \infty$, since the expectation of a uniformly bounded variable is bounded. This implies existence of the first moment. 

\subsection{MDS Central Limit Theorem}

\cite{kato2021} used a MDS CLT which requires (condition b) a finite $2 + \delta$ moment ($\delta >0$) for $ |z_t| $. Instead we use the MDS CLT as stated by \cite{dvoretzky1972}. This statement contains a Lindeberg type condition where we must only consider the second moment of $ |z_t| $. Since we do not assume boundedness, we opt for this Lindeberg-type statement. For completeness, we present this theorem as it is stated in \citet[Theorem 2]{zhang2021}.

\begin{theorem}[MDS Central Limit Theorem] \nonumber

Let ${Z_T (\mathcal{P})}_{T \geq 1}$ be a sequence of random variables whose distributions are defined by some $\mathcal{P} \in \mathbb{P}$ and some nuisance component $\eta$. Moreover, let ${Z_T (\mathcal{P})}_{T \geq 1}$ be a martingale difference sequence with respect to $\Omega_t$, meaning
$\mathbb{E}_{\mathcal{P}, \eta}[Z_t(\mathcal{P}) \mid \Omega_{t-1}] = 0$ for all $t \geq 1$ and $\mathcal{P} \in \mathbb{P}$. If we assume that,

\begin{enumerate}
    \item $\frac{1}{T} \sum_{t= 1}^{T} \mathbb{E}_{\mathcal{P},\eta} \left[  z_t^2  \mid   \Omega_{t-1} \right] \xrightarrow[]{p} \sigma^2$ uniformly over $\mathcal{P} \in \mathbb{P}$, where $\sigma^2$ is a constant $0<\sigma^2 < \infty$, and that,
    \item for any $\epsilon > 0$, $\frac{1}{T} \sum_{t= 1}^{T} \mathbb{E}_{\mathcal{P}, \eta} \left[ z_t(\mathcal{P})^2 \mathbbm{1}\left[ |z_t(\mathcal{P})| > \epsilon] \mid  \Omega_{t-1} \right]\right] \xrightarrow[]{p} 0$ uniformly over $\mathcal{P} \in \mathbb{P}$,
\end{enumerate}
then $\sqrt{T}(\Bar{z_t}) \xrightarrow[]{d} N(0, \sigma^2)$ uniformly over $\mathcal{P} \in \mathbb{P}$.

\end{theorem}

Dropping the requirement of the conditions holding uniformly over $\mathcal{P} \in \mathbb{P}$ recovers the original result by \cite{dvoretzky1972}. Below we show that these two conditions are satisfied. It follows that $$\sqrt{T}(\Bar{z_t}) = \frac{\hat{\theta}^{\mathrm{A2IPW}} - \theta_0}{\sqrt{T}}\xrightarrow[]{d} N(0, \sigma^2),$$ where
$$\sigma^2 = \mathbb{E} \left[ \sum_{a=0}^1 \frac{v(a, X_t)}{\pi(a \mid X_t)} + \left( f(1,X_t) - f(0,X_t) - \theta_0 \right)^2 \right].$$
\newline

\subsubsection{Condition 1 (Conditional Variance)} 
We wish to show that $$\frac{1}{T} \sum_{t= 1}^{T} \mathbb{E} \left[ z_t^2 \mid  \Omega_{t-1} \right] \xrightarrow[]{p} \sigma^2 = \mathbb{E}\left[\sum^{1}_{a=0}\frac{v\big(a, X_t\big)}{\pi(a\mid X_t)} + \Big(f(1,X_t) - f(0,X_t) - \theta_0\Big)^2\right].$$ This is equivalent to showing $$\frac{1}{T} \sum_{t= 1}^{T} \left( \mathbb{E} \left[ z_t^2 \mid  \Omega_{t-1} \right] - \sigma^2 \right) \xrightarrow[]{p} 0.$$ To reduce notational clutter, let $\mathbb{E}(X_t \mid \Omega_{t-1} )$ be denoted as $\mathbb{E}^{t-1}(X_t)$. \citet[Appendix~B]{kato2021} show
\begin{align}
     &\mathbb{E} \left[ z_t^2 \mid  \Omega_{t-1} \right] - \sigma^2 = \mathbb{E}^{t-1} \Bigg[ 
   \frac{(Y_t(1) - \hat{f}_{t-1}(1,X_t))^2}{\pi_t(1 \mid X_t, \Omega_{t-1})} +
   \frac{(Y_t(0) - \hat{f}_{t-1}(0,X_t))^2}{\pi_t(0 \mid X_t, \Omega_{t-1})} \nonumber \\  &+ 
   \left(\hat{f}_{t-1}(1,X_t) - \hat{f}_{t-1}(0,X_t) - \theta_0\right)^2 \nonumber \\&+
   2(f(1,X_t) - f(0,X_t) - \hat{f}_{t-1}(1,X_t) + \hat{f}_{t-1}(0,X_t) )(\hat{f}_{t-1}(1,X_t) - \hat{f}_{t-1}(0,X_t) -\theta_0) \Bigg] \nonumber \\ &- 
   \mathbb{E}^{t-1} \left[ \frac{(Y_t(1) - f(1,X_t))^2}{\pi(1 \mid X_t)} +
   \frac{(Y_t(0) - f(0,X_t))^2}{\pi(0 \mid X_t)} + (f(1,X_t) - f(0,X_t) - \theta_0)^2
   \right] \nonumber \\ 
   &= \sum_{a = 0}^1 \mathbb{E}^{t-1}\left[ \frac{\left( Y_t(a) - \hat{f}_{t-1}(a, X_t)\right)^2}{\pi_t(a \mid X_t, \Omega_{t-1})} - \frac{\left( Y_t(a) - f(a, X_t)\right)^2}{\pi(a \mid X_t)} \right] \label{eq:term1mds}\\ &+ 
   2 \mathbb{E}^{t-1} \left[ \left(f(1,X_t) - f(0,X_t) - \hat{f}_{t-1}(1,X_t) + \hat{f}_{t-1}(0,X_t) \right) \left(\hat{f}_{t-1}(1,X_t) - \hat{f}_{t-1}(0,X_t) -\theta_0 \right) \right] \label{eq:term2mds}\\ &+ \mathbb{E}^{t-1} \left[ \left(\hat{f}_{t-1}(1,X_t) - \hat{f}_{t-1}(0,X_t) - \theta_0\right)^2 - \left( f(1,X_t) - f(0,X_t) - \theta_0 \right)^2 \right] \label{eq:term3mds} .
\end{align}

We now consider terms \eqref{eq:term1mds}, \eqref{eq:term2mds} and \eqref{eq:term3mds} individually. We make use of auxiliary lemmas and defer proofs to Appendix~\ref{appdx:auxiliary_proofs}. In all of the lemmas below, we keep all assumptions from Appendix~\ref{appdx:assumptions}.

\begin{lemma}[Convergence of ~\eqref{eq:term1mds}]\label{lemma:mds1} Under the assumptions of Theorem~\ref{theorem:MDS-CLT}, we have
    $$\sum_{a = 0}^1 \mathbb{E}^{t-1}\left[ \frac{\left( Y_t(a) - \hat{f}_{t-1}(a, X_t)\right)^2}{\pi_t(a \mid X_t, \Omega_{t-1})} - \frac{\left( Y_t(a) - f(a, X_t)\right)^2}{\pi(a \mid X_t)} \right] = o_{\mathbb{P}}(1).$$
\end{lemma}
\begin{lemma}[Convergence of ~\eqref{eq:term2mds}]\label{lemma:mds2} Under the assumptions of Theorem~\ref{theorem:MDS-CLT}, we have
    $$2 \mathbb{E}^{t-1} \left[ \left(f(1,X_t) - f(0,X_t) - \hat{f}_{t-1}(1,X_t) + \hat{f}_{t-1}(0,X_t) \right) \left(\hat{f}_{t-1}(1,X_t) - \hat{f}_{t-1}(0,X_t) -\theta_0 \right) \right]= o_{\mathbb{P}}(1).$$
\end{lemma}
\begin{lemma}[Convergence of ~\eqref{eq:term3mds}]\label{lemma:mds3} Under the assumptions of Theorem~\ref{theorem:MDS-CLT}, we have
    $$\mathbb{E}^{t-1} \left[ \left(\hat{f}_{t-1}(1,X_t) - \hat{f}_{t-1}(0,X_t) - \theta_0\right)^2 - \left( f(1,X_t) - f(0,X_t) - \theta_0 \right)^2 \right]= o_{\mathbb{P}}(1).$$
\end{lemma} 
 
 Given Lemmas~\ref{lemma:mds1},~\ref{lemma:mds2}, and ~\ref{lemma:mds3}, convergence in probability to zero for each term is established, and therefore so is the convergence of the sum. Convergence of the conditional variance of the MDS is then established. So far we have shown that $ \mathbb{E}^{t-1}\left[ z_t^2\right] \xrightarrow[]{p} \sigma^2$, but we wish to show that $\frac{1}{T}\sum_{t=1}^T \left( \mathbb{E}^{t-1}\left[ z_t^2\right] - \sigma^2\right) \xrightarrow[]{p} 0$.

Let $a_t = \mathbb{E}^{t-1}\left[z_t^2\right]$ and $a = \sigma^2$. It follows that $\mathbb{E} [a_t] = \mathbb{E}[|a_t|]< M < \infty$ for all t and some constant $M$, where the equality holds since $a_t > 0$. This uniform boundedness implies that $a_t$ is uniformly integrable. Since it is also true that $a_t \xrightarrow[]{p} a$, then by the $L^R$ convergence theorem, we have that $a_t \to a$ in $L^1$, which implies $\mathbb{E}\left[|\mathbb{E}^{t-1}[z_t^2] - \sigma^2|\right] \to 0 $ \citep{loeve1977probability}. 

\cite{kato2021} show through Markov's inequality that  $\frac{1}{T}\sum_{t=1}^T \left( \mathbb{E}^{t-1}\left[ z_t^2\right] - \sigma^2\right) \xrightarrow[]{p} 0$ if $\mathbb{E}\left[|\mathbb{E}^{t-1}[z_t^2] - \sigma^2|\right] \to 0 $. So condition 1 is satisfied.

\subsubsection{Condition 2 (Conditional Lindeberg)}\label{sec:lindeberg}

We seek to show that for any $\delta > 0$, 
$$\frac{1}{T} \sum_{t=1}^{T} \mathbb{E} \left( z_t^2 \mathbbm{1}\left[ |z_t|  > \delta \sqrt{T}\right] ~\Big|~ \Omega_{t-1} \right)\xrightarrow[]{p} 0.$$

Define $b_t = z_t^2 \mathbbm{1}( |z_t|  > \delta \sqrt{T})$. Then $b_t = z_t^2$ w.p. $\mathbb{P}( |z_t|  > \delta \sqrt{T})$ and $0$ otherwise. By Chebyshev's inequality,
$$\mathbb{P} ( |z_t|  > \delta \sqrt{T}) \leq \frac{\text{Var}(z_t)}{\delta^2 T} .$$
We note that $\mathrm{Var}(z_t) = \mathbb{E}(z_t^2) < \infty$. This gives
$$\lim_{T \to \infty}\frac{\text{Var}(z_t)}{\delta^2 T} = 0,$$
which implies that $b_t \xrightarrow[]{p} 0$, and $b_t \xrightarrow[]{d} 0$. 

Note that $ |b_t|  \leq z_t^2 $, and $\mathbb{E}(z_t^2)  < \infty$. By the dominated convergence theorem, $\lim_{T \to \infty}\mathbb{E}(b_t) = \mathbb{E} (\lim_{T \to \infty} b_t) = 0$. Hence, we have
$$\frac{1}{T} \sum_{t=1}^{T} \mathbb{E } \left( z_t^2 \mathbbm{1}\left[ |z_t|  > \delta \sqrt{T}\right] ~\Big|~ \Omega_{t-1} \right) \xrightarrow[]{p} 0.$$

\section{Auxiliary Lemmas and Proofs}\label{appdx:auxiliary_proofs}
This appendix shows proofs for auxiliary lemmas used in Appendix~\ref{appdx:mds_clt_proof}. The proofs involve tedious algebra and are included in full detail for completeness.

\subsection{Proof of Lemma~\ref{lemma:mds1}}
The term considered in Lemma~\ref{lemma:mds1}, specifically term~\eqref{eq:term1mds}, involves a summation over the potential treatments, we choose to focus on a single arbitrary treatment, $a$, and show that the term for an individual treatment converges to zero in probability, and hence, so does the sum. 

\begin{align}
     \mathbb{E}^{t-1} &\left[ \frac{\left( Y_t(a) - \hat{f}_{t-1}(a, X_t)\right)^2}{\pi_t(a \mid X_t, \Omega_{t-1})} - \frac{\left( Y_t(a) - f(a, X_t)\right)^2}{\pi(a \mid X_t)}  \right] \nonumber \\
     &= \mathbb{E}^{t-1} \left[ \frac{\left( Y_t(a) - \hat{f}_{t-1}(a, X_t) + f(a, X_t) - f(a, X_t)\right)^2}{\pi_t(a \mid X_t, \Omega_{t-1})} - \frac{\left( Y_t(a) - f(a, X_t)\right)^2}{\pi(a \mid X_t)} \right] \label{eq:pm_f}\\
     &= \mathbb{E}^{t-1} \left[ \frac{\left( \left( Y_t(a) - f(a, X_t) \right) + \left( f(a, X_t) -\hat{f}_{t-1}(a, X_t) \right) \right)^2}{\pi_t(a \mid X_t, \Omega_{t-1})} - \frac{\left( Y_t(a) - f(a, X_t)\right)^2}{\pi(a \mid X_t)} \right] \label{eq:arrange_f}\\
     &= \mathbb{E}^{t-1} \Bigg[ \frac{ \left( Y_t(a) - f(a, X_t) \right)^2}{\pi_t(a \mid X_t, \Omega_{t-1})} + \frac{2\left(Y_t(a)- f(a, X_t)\right) \left(f(a, X_t) - \hat{f}_{t-1}(a, X_t) \right)^2 }{\pi_t(a \mid X_t, \Omega_{t-1})} \label{eq:expand_factor} \\ &+ \frac{\left( f(a, X_t) - \hat{f}_{t-1}(a, X_t)\right)^2}{\pi_t(a \mid X_t, \Omega_{t-1})} - \frac{\left( Y_t(a) - f(a, X_t)\right)^2}{\pi(a \mid X_t)} \Bigg] \nonumber \\
     &= \mathbb{E}^{t-1} \left[ \left( Y_t(a) - f(a, X_t) \right)^2 \left( \frac{1}{\pi_t(a \mid X_t, \Omega_{t-1})} - \frac{1}{\pi(a \mid X_t)} \right) \right] \label{eq:cancel_fl} \\ &+ 2 \mathbb{E}^{t-1} \left[\frac{\left(Y_t(a)- f(a, X_t)\right) \left(f(a, X_t) - \hat{f}_{t-1}(a, X_t) \right)^2 }{\pi_t(a \mid X_t, \Omega_{t-1})} + \frac{\left( f(a, X_t) - \hat{f}_{t-1}(a, X_t)\right)^2}{\pi_t(a \mid X_t, \Omega_{t-1})} \right] .\nonumber 
    \nonumber
     \end{align}
     
     Above,~\eqref{eq:pm_f} simultaneously adds and subtracts $f(a,X_t)$, while \eqref{eq:arrange_f} and \eqref{eq:expand_factor} square the binomial term. In equation~\eqref{eq:cancel_fl}, we factor $(Y_t(a) - f(a,X_t))^2$ from the first and final terms, and utilize linearity of expectation. Continuing,
     \begin{align}
     &\mathbb{E}^{t-1} \left[ \left( Y_t(a) - f(a, X_t) \right)^2 \left( \frac{1}{\pi_t(a \mid X_t, \Omega_{t-1})} - \frac{1}{\pi(a \mid X_t)} \right) \right] \nonumber \\ &+ 2 \mathbb{E}^{t-1} \left[\frac{\left(Y_t(a)- f(a, X_t)\right) \left(f(a, X_t) - \hat{f}_{t-1}(a, X_t) \right)^2 }{\pi_t(a \mid X_t, \Omega_{t-1})} + \frac{\left( f(a, X_t) - \hat{f}_{t-1}(a, X_t)\right)^2}{\pi_t(a \mid X_t, \Omega_{t-1})} \right] \nonumber \\
    &\leq \mathbb{E}^{t-1} \left[ \left( Y_t(a) - f(a, X_t) \right)^2 \left( \frac{1}{\pi_t(a \mid X_t, \Omega_{t-1})} - \frac{1}{\pi(a \mid X_t)} \right) \right] \label{eq:policy_trunc_ineq} \\ &+ 2 \mathbb{E}^{t-1} \left[\frac{\left(Y_t(a)- f(a, X_t)\right) \left(f(a, X_t) - \hat{f}_{t-1}(a, X_t) \right)^2 }{\pi_t(a \mid X_t, \Omega_{t-1})}\right] + \mathbb{E}^{t-1} \left[ k_t \left( f(a, X_t) - \hat{f}_{t-1}(a, X_t)\right)^2\right] \nonumber\\
     &= \mathbb{E}^{t-1} \left[ \left( Y_t(a) - f(a, X_t) \right)^2 \left( \frac{1}{\pi_t(a \mid X_t, \Omega_{t-1})} - \frac{1}{\pi(a \mid X_t)} \right) \right] \label{eq:substiture_l2normdef}\\ &+ 2 \mathbb{E}^{t-1} \left[\frac{\left(Y_t(a)- f(a, X_t)\right) \left(f(a, X_t) - \hat{f}_{t-1}(a, X_t) \right)^2 }{\pi_t(a \mid X_t, \Omega_{t-1})}\right]+ k_t \left(\lVert \hat{f}(a, X_t) - f(a, X_t))\rVert_2 \right)^2 \nonumber .
    \nonumber
    \end{align}
Inequality~\eqref{eq:policy_trunc_ineq} follows since our policy is truncated within $[\frac{1}{k_t}, 1 - \frac{1}{k_t}]$. Equation~\eqref{eq:substiture_l2normdef} uses norm notation in the final term so that we may reference our assumptions further in the proof. We now turn our focus to simplifying the middle term of equation~\eqref{eq:substiture_l2normdef},
    \begin{align}
     &\mathbb{E}^{t-1} \left[ \left( Y_t(a) - f(a, X_t) \right)^2 \left( \frac{1}{\pi_t(a \mid X_t, \Omega_{t-1})} - \frac{1}{\pi(a \mid X_t)} \right) \right] \nonumber \\ &+ 2 \mathbb{E}^{t-1} \left[\frac{\left(Y_t(a)- f(a, X_t)\right) \left(f(a, X_t) - \hat{f}_{t-1}(a, X_t) \right)^2 }{\pi_t(a \mid X_t, \Omega_{t-1})}\right]+ k_t \left(\lVert \hat{f}(a, X_t) - f(a, X_t))\rVert_2 \right)^2 \nonumber\\
     &= \mathbb{E}^{t-1} \left[ \left( Y_t(a) - f(a, X_t) \right)^2 \left( \frac{1}{\pi_t(a \mid X_t, \Omega_{t-1})} - \frac{1}{\pi(a \mid X_t)} \right) \right] \label{eq:expand_middle}\\ &+ 2 \left( \mathbb{E}^{t-1} \left[ \frac{Y_t(a)}{\pi_t(a \mid X_t, \Omega_{t-1})}\left(f(a, X_t) - \hat{f}(a, X_t) \right)^2 \right] - \mathbb{E}^{t-1} \left[ \frac{f(a, X_t)}{\pi_t(a \mid X_t, \Omega_{t-1})} \left(\hat{f}(a, X_t) - f(a, X_t) \right)^2 \right] \right)\nonumber\\ &+ k_t \left(\lVert \hat{f}(a) - f(a))\rVert_2 \right)^2\nonumber\\
    &=\mathbb{E}^{t-1} \left[ \left( Y_t(a) - f(a, X_t) \right)^2 \left( \frac{1}{\pi_t(a \mid X_t, \Omega_{t-1})} - \frac{1}{\pi(a \mid X_t)} \right) \right] \nonumber \\ &+ 2 \Bigg( \mathbb{E}^{t-1} \left[ \mathbb{E} \left[ \frac{Y_t(a)}{\pi_t(a \mid X_t, \Omega_{t-1})}\left(f(a, X_t) - \hat{f}(a, X_t) \right)^2 ~\Big|~ X_t, \Omega_{t-1} \right] \right] \label{eq:iteratedexpectation_middle1} \\ &- \mathbb{E}^{t-1} \left[ \mathbb{E} \left[ \frac{f(a, X_t)}{\pi_t(a \mid X_t, \Omega_{t-1})} \left(\hat{f}(a, X_t) - f(a, X_t) \right)^2 ~\Big|~ X_t, \Omega_{t-1} \right] \right] \Bigg) \label{eq:iteratedexpectation_middle2} \\& + k_t \left(\lVert \hat{f}(a, X_t) - f(a, X_t))\rVert_2 \right)^2 \nonumber , \nonumber
     \end{align}
where equation~\eqref{eq:expand_middle} expands the term of interest, and terms~\eqref{eq:iteratedexpectation_middle1} and \eqref{eq:iteratedexpectation_middle2} apply the law of iterated expectation. Conditioning on $X_t$, the only non-constant term in terms ~\eqref{eq:iteratedexpectation_middle1} and \eqref{eq:iteratedexpectation_middle2} is $Y_t(a)$, whose conditional expectation on $X_t$ is $f(a,X_t)$. Therefore, the terms ~\eqref{eq:iteratedexpectation_middle1} and \eqref{eq:iteratedexpectation_middle2} reduce to 0. Simplifying, we have
     \begin{align}
     &\mathbb{E}^{t-1} \left[ \left( Y_t(a) - f(a, X_t) \right)^2 \left( \frac{1}{\pi_t(a \mid X_t, \Omega_{t-1})} - \frac{1}{\pi(a \mid X_t)} \right) \right] \nonumber \\ &+ 2 \Bigg( \mathbb{E}^{t-1} \left[ \mathbb{E} \left[ \frac{Y_t(a)}{\pi_t(a \mid X_t, \Omega_{t-1})}\left(f(a, X_t) - \hat{f}(a, X_t) \right)^2 ~\Big|~ X_t, \Omega_{t-1} \right] \right] \nonumber \\ &- \mathbb{E}^{t-1} \left[ \mathbb{E} \left[ \frac{f(a, X_t)}{\pi_t(a \mid X_t, \Omega_{t-1})} \left(\hat{f}(a, X_t) - f(a, X_t) \right)^2 ~\Big|~ X_t, \Omega_{t-1} \right] \right] \Bigg) \nonumber \\& + k_t \left(\lVert \hat{f}(a, X_t) - f(a, X_t))\rVert_2 \right)^2 \nonumber \\ 
&= \mathbb{E}^{t-1} \left[ \left( Y_t(a) - f(a, X_t) \right)^2 \left( \frac{1}{\pi_t(a \mid X_t, \Omega_{t-1})} - \frac{1}{\pi(a \mid X_t)} \right) \right] \label{eq:nosecondterm} \\ &+ k_t \left(\lVert \hat{f}(a,X_t) - f(a, X_t))\rVert_2 \right)^2 .\nonumber 
\end{align}

By assumption, $k_t \lVert\hat{f}(a,X_t) - f(a,X_t)\rVert_2 = o_{\mathbb{P}}(1)$. It follows then that $k_t \left(\lVert \hat{f}(a,X_t) - f(a, X_t))\rVert_2 \right)^2 = o_{\mathbb{P}}(1)$. Equation~\eqref{eq:nosecondterm} can be further simplified to
\begin{align}
    &\mathbb{E}^{t-1} \left[ \left( Y_t(a) - f(a, X_t) \right)^2 \left( \frac{1}{\pi_t(a \mid X_t, \Omega_{t-1})} - \frac{1}{\pi(a \mid X_t)} \right) \right] \nonumber \\ &+ k_t \left(\lVert \hat{f}(a,X_t) - f(a, X_t))\rVert_2 \right)^2 \nonumber \\ 
    &= \mathbb{E}^{t-1} \left[ \left( Y_t(a) - f(a, X_t) \right)^2 \left( \frac{1}{\pi_t(a \mid X_t, \Omega_{t-1})} - \frac{1}{\pi(a \mid X_t)} \right) \right] + o_\mathbb{P}(1) \nonumber \\
    &= \mathbb{E} \left[ \mathbb{E} \left[ \frac{\pi(a \mid X_t) - \pi_t(a \mid X_t, \Omega_{t-1})}{\pi(a \mid X_t)\pi_t(a \mid X_t, \Omega_{t-1})} \left(Y_t(a) - f(a, X_t)\right)^2 ~\Big|~ X_t, \Omega_{t-1}\right] ~\Big|~ \Omega_{t-1}\right] + o_\mathbb{P}(1) \label{eq:total_expec}\\
    &= \mathbb{E} \left[ \frac{\pi(a \mid X_t) - \pi_t(a \mid X_t, \Omega_{t-1})}{\pi(a \mid X_t)\pi_t(a \mid X_t, \Omega_{t-1})} \mathbb{E} \left[ \left(Y_t(a) - f(a, X_t)\right)^2 ~\Big|~ X_t, \Omega_{t-1}\right] ~\Big|~ \Omega_{t-1} \right] + o_\mathbb{P}(1)\label{eq:total_indep} \\
    &\leq C_1 C_2 k_t \mathbb{E}^{t-1} \left[\pi(a \mid X_t) - \pi_t(a \mid X_t, \Omega_{t-1}) \right]+o_\mathbb{P}(1) = o_\mathbb{P}(1) \label{eq:substitute_condvar}.
\end{align}
Equation~\eqref{eq:total_expec} follows from the law of total expectation. In equation~\eqref{eq:total_indep} $\pi_t(a \mid X_t, \Omega_{t-1})$ given $X_t$ and $\Omega_{t-1}$ is constant, and can be moved out of the inner expectation, away from $(Y_t(a) - f(a,X_t))^2$. The bound~\eqref{eq:substitute_condvar} then utilizes our policy truncation and our assumption that $\frac{1}{\pi} $ is uniformly bounded by $C_1$. We are able to bound the denominator with a constant, and move this constant outside of the expectation. Simultaneously, we note that the inner expectation is by definition $v(a,x)$. By assumption, $v(a,x)$ is bounded uniformly by $C_2 < \infty$. The bound~\ref{eq:substitute_condvar} reduces to $o_{\mathbb{P}}(1)$, since convergence in $\ell_2$ implies convergence in $\ell_1$, and the Lemma is proved.

\subsection{Proof of Lemma~\ref{lemma:mds2}}
We look to prove that
    $$2 \mathbb{E}^{t-1} \left[ \left(f(1,X_t) - f(0,X_t) - \hat{f}_{t-1}(1,X_t) + \hat{f}_{t-1}(0,X_t) \right) \left(\hat{f}_{t-1}(1,X_t) - \hat{f}_{t-1}(0,X_t) -\theta_0 \right) \right]= o_{\mathbb{P}}(1).$$
For simplicity, we temporarily ignore the constant. Continuing, 
\begin{align}
\mathbb{E}^{t-1} &\left[ \left(f(1,X_t) - f(0,X_t) -\hat{f}_{t-1}(1,X_t) + \hat{f}_{t-1}(0,X_t) \right) \left(f(1,X_t) - f(0,X_t) - \theta_0 \right)\right] \nonumber \\ 
&= \mathbb{E}^{t-1} \left[ \left(f(1, X_t) - \hat{f}_{t-1}(1,X_t) \right) \left(f(1, X_t) - f(0,X_t) - \theta_0\right) \right] \label{eq:binomial_term2}\\ &\hspace{10mm}+ \mathbb{E}^{t-1} \left[ \left( f(0,X_t) - \hat{f}(0,X_t) \right)\left(f(1,X_t) - f(0, X_t) - \theta_0\right)  \right] \nonumber \\
&\leq \sqrt{\mathbb{E}^{t-1} \left[\left(f(1, X_t) - \hat{f}_{t-1}(1,X_t) \right)^2 \right] \mathbb{E}^{t-1} \left[ \left(f(1, X_t) - f(0,X_t) - \theta_0\right)^2 \right] } \label{eq:cauchyschwartz_mds2}\\
&\hspace{10mm}+ \sqrt{\mathbb{E}^{t-1} \left[\left( f(0,X_t) - \hat{f}(0,X_t) \right)^2 \right]\mathbb{E}^{t-1} \left[ \left(f(1, X_t) - f(0,X_t) - \theta_0\right)^2 \right]} , \nonumber 
\end{align}
where equation ~\eqref{eq:binomial_term2} separates terms from different treatments and utilizes the linearity of expectation. Bound~\eqref{eq:cauchyschwartz_mds2} then follows from applying the Cauchy-Schwarz inequality. We conclude by showing 
\begin{align*}
&\sqrt{\mathbb{E}^{t-1} \left[\left(f(1, X_t) - \hat{f}_{t-1}(1,X_t) \right)^2 \right] \mathbb{E}^{t-1} \left[ \left(f(1, X_t) - f(0,X_t) - \theta_0\right)^2 \right] } \\
&\hspace{10mm}+ \sqrt{\mathbb{E}^{t-1} \left[\left( f(0,X_t) - \hat{f}(0,X_t) \right)^2 \right]\mathbb{E}^{t-1} \left[ \left(f(1, X_t) - f(0,X_t) - \theta_0\right)^2 \right]} \nonumber \\
&= \lVert \hat{f} - f \rVert_2 \sqrt{\mathbb{E}^{t-1} \left[ \left(f(1, X_t) - f(0,X_t) - \theta_0\right)^2 \right]} + \lVert \hat{f} - f \rVert_2 \sqrt{\mathbb{E}^{t-1} \left[ \left(f(1, X_t) - f(0,X_t) - \theta_0\right)^2 \right]}\nonumber\\
&= o_\mathbb{P}(1). \nonumber
\end{align*}
 Using norm notation and applying the assumption of convergence of regression in $\ell_2$-norm, convergence is established, and the lemma is proved.

\subsection{Proof of Lemma~\ref{lemma:mds3}}
We wish to prove that
$$\mathbb{E}^{t-1} \left[ \left(\hat{f}_{t-1}(1,X_t) - \hat{f}_{t-1}(0,X_t) - \theta_0\right)^2 - \left( f(1,X_t) - f(0,X_t) - \theta_0 \right)^2 \right]= o_{\mathbb{P}}(1).$$
We begin by expanding this term,
\begin{align}
   \mathbb{E}^{t-1} &\left[ \left(\hat{f}_{t-1}(1,X_t) - \hat{f}_{t-1}(0,X_t) - \theta_0\right)^2 - \left( f(1,X_t) - f(0,X_t) - \theta_0 \right)^2 \right] \nonumber \\
    & = \mathbb{E}^{t-1} \Bigg[ \left( \left(\hat{f}(1,X_t) - \hat{f}(0, X_t) - \theta_0\right) + \left( f(1,X_t) - f(0,X_t) - \theta_0 \right) \right)\label{eq:a_minus_b}\\ &\hspace{20mm}\times \left(\left(\hat{f}(1,X_t) - \hat{f}(0, X_t) - \theta_0\right) - \left( f(1,X_t) - f(0,X_t) - \theta_0 \right)\right)  \Bigg] \nonumber \\
&= \mathbb{E}^{t-1} \Bigg[ \left((\hat{f}(1,X_t) + f(1,X_t)) - (f(0,X_t) +\hat{f}(0, X_t)) - 2\theta_0\right) \label{eq:reduce_theta}\\ &\hspace{20mm}\times \left(\left(\hat{f}(1,X_t) - f(1,X_t) \right) + \left( f(0,X_t) - \hat{f}(0, X_t)\right) \right)  \Bigg]. \nonumber 
\end{align}
Equation~\eqref{eq:a_minus_b} arises from the fact that $a^2 - b^2 = (a+b)(a-b)$ for real numbers $a, b$. Equation~\eqref{eq:reduce_theta} then collapses $\theta_0$ to a single term. We now add and subtract $f(1,X_t)$ and $f(0,X_t)$ to the first term of equation~\eqref{eq:reduce_theta}, giving
\begin{align}
&\mathbb{E}^{t-1} \Bigg[ \left((\hat{f}(1,X_t) + f(1,X_t)) - (f(0,X_t) +\hat{f}(0, X_t)) - 2\theta_0\right) \nonumber \\ &\hspace{20mm}\times \left(\left(\hat{f}(1,X_t) - f(1,X_t) \right) + \left( f(0,X_t) - \hat{f}(0, X_t)\right) \right)  \Bigg] \nonumber \\ 
&= \mathbb{E}^{t-1} \Bigg[\left( (\hat{f}(1,X_t) - f(1,X_t)) + (f(0,X_t) - \hat{f}(0, X_t)) + 2\left(f(1,X_t) - f(0,X_t) - \theta_0 \right) \right)\label{eq:f_differences}\\ &\hspace{20mm}\times \left(\left(\hat{f}(1,X_t) - f(1,X_t) \right) + \left( f(0,X_t) - \hat{f}(0, X_t)\right) \right)  \Bigg]. \nonumber 
\end{align}
Equation~\eqref{eq:f_differences} completes this step, and rearranges terms so that we may use the assumption of convergence of regression.

Next, distributing the second term in equation~\eqref{eq:f_differences}, along with the use of the linearity of expectation gives
\begin{align}
     \mathbb{E}^{t-1} &\Bigg[\left( (\hat{f}(1,X_t) - f(1,X_t)) + (f(0,X_t) - \hat{f}(0, X_t)) +2\left(f(1,X_t) - f(0,X_t) - \theta_0 \right) \right) \nonumber\\ &\hspace{20mm}\times \left(\left(\hat{f}(1,X_t) - f(1,X_t) \right) + \left( f(0,X_t) - \hat{f}(0, X_t)\right) \right)  \Bigg] \nonumber\\
     &= \mathbb{E}^{t-1} \Bigg[\left( (\hat{f}(1,X_t) - f(1,X_t)) + (f(0,X_t) - \hat{f}(0, X_t)) \right)^2 \Bigg] \nonumber \\ &+ 2 \mathbb{E}^{t-1} \Bigg[\left(f(1,X_t) - f(0,X_t) - \theta_0 \right) \left(\left(\hat{f}(1,X_t) - f(1,X_t) \right) + \left( f(0,X_t) - \hat{f}(0, X_t)\right) \right)\Bigg] \label{eq:f_difference_distributed}.
     \end{align}
The first term in equation~\eqref{eq:f_difference_distributed} converges in probability by assumption. For the second term, we distribute $f(1,X_t) - f(0,X_t) - \theta_0)$ yielding 
     \begin{align}
     &\mathbb{E}^{t-1} \Bigg[\left( (\hat{f}(1,X_t) - f(1,X_t)) +  (f(0,X_t) - \hat{f}(0, X_t)) \right)^2 \Bigg] \\  &+ 2 \mathbb{E}^{t-1} \Bigg[\left(f(1,X_t) - f(0,X_t)  - \theta_0 \right) \left(\left(\hat{f}(1,X_t) - f(1,X_t) \right) + \left( f(0,X_t) - \hat{f}(0, X_t)\right)  \right) \nonumber \\
     &= \mathbb{E}^{t-1} \left[\left( 2\left(f(1,X_t) - f(0,X_t)  - \theta_0 \right) \right)\left( f(1,X_t) - \hat{f}(1, X_t)\right) \right]\\ &\hspace{10mm}+ \mathbb{E}^{t-1} \left[\left( 2\left(f(1,X_t) - f(0,X_t)  - \theta_0 \right) \right)\left( f(0,X_t) - \hat{f}(0, X_t)\right) \right] + o_\mathbb{P}(1) \label{eq:dist_2terms}.\
\end{align}
Applying the Cauchy-Schwarz inequality to each term in equation~\eqref{eq:dist_2terms} gives
\begin{align}
    &\mathbb{E}^{t-1} \left[\left( 2\left(f(1,X_t) - f(0,X_t)  - \theta_0 \right) \right)\left( f(1,X_t) - \hat{f}(1, X_t)\right) \right]\\ &\hspace{10mm}+ \mathbb{E}^{t-1} \left[\left( 2\left(f(1,X_t) - f(0,X_t)  - \theta_0 \right) \right)\left( f(0,X_t) - \hat{f}(0, X_t)\right) \right] + o_\mathbb{P}(1) \nonumber \\
    &\leq \sqrt{\mathbb{E}^{t-1} \left[\left( 2\left(f(1,X_t) - f(0,X_t)  - \theta_0 \right) \right)^2\right] \mathbb{E}^{t-1}\left[\left( f(1,X_t) - \hat{f}(1, X_t)\right)^2 \right]} \label{eq:cauchy_term3}\\ &\hspace{10mm}+ \sqrt{\mathbb{E}^{t-1} \left[\left( 2\left(f(1,X_t) - f(0,X_t)  - \theta_0 \right) \right)^2\right]\mathbb{E}^{t-1} \left[\left( f(0,X_t) - \hat{f}(0, X_t)\right)^2 \right]} + o_\mathbb{P}(1) \nonumber \\
    &= \sqrt{\mathbb{E}^{t-1} \left[\left( 2\left(f(1,X_t) - f(0,X_t)  - \theta_0 \right) \right)^2\right]} \lVert \hat{f} - f \rVert_2 \label{eq:2subs_convergence} \\ &\hspace{10mm}+ \sqrt{\mathbb{E}^{t-1} \left[\left( 2\left(f(1,X_t) - f(0,X_t)  - \theta_0 \right) \right)^2\right]}\lVert \hat{f} - f \rVert_2  + o_\mathbb{P}(1) \nonumber.
\end{align}

Equation~\eqref{eq:2subs_convergence} follows from the bound~\eqref{eq:cauchy_term3} by definition. Since $\lVert \hat{f} - f \rVert_2 = o_{\mathbb{P}}(1)$, Equation~\eqref{eq:2subs_convergence} reduces to $o_{\mathbb{P}}(1)$, and the lemma is proved.

\section{Proof of Theorem~\ref{theorem:bettingCS}}\label{appdx:nonasymp_proof}

\subsection{Proof Outline}
The proof \emph{adapts} the proof of \citet[Theorem 1]{waudbysmith2022anytimevalid} to our problem setting. The only departure in our proof is that our parameter space and $(\lambda_t)_{t=1}^{T}$ are not strictly non-negative. We include this proof to demonstrate how our bounds on $\lambda_t$ originate as well as showing how our proof does not make use of the mirroring technique to form a $(1-\alpha)$-upper CS. Although these adjustments are immediate and obvious to those familiar with the anytime-valid inference literature, we include this proof for completeness. We begin by stating and proving a lemma that demonstrates how to construct an arbitrary $(1-\alpha)$ Betting-CS for our problem setting. We then construct a Hedged-CS, where we specify the capital process, the convex combination, relevant user-specified parameters and invoke our adapted lemma.

\subsection{Constructing a $(1-\alpha)$ Betting-CS}
\begin{lemma}\label{lemma:bettingCS}
    Assume we observe data following the data generating process of Section~\ref{sec:dgp}. Assume that $Y_t \in [0,1]$ $\forall t \in 1,\dots, T.$ Suppose that $\pi_t(1 \mid X_t, \Omega_{t-1}) \in [\frac{1}{k_t}, 1-\frac{1}{k_t}]$ for all $ t \in 1,\dots, T,$ then 
    \begin{equation*}
        C_T^{\text{Betting}} := \bigcap_{t \leq T} \left\{ \theta^{'} \in [-1,1] : \prod_{t=1}^{T} \left(1 + \lambda_t(\theta^{'})(h_t - \theta^{'} )\right) < \frac{1}{\alpha}  \right\},
    \end{equation*}
    forms a $(1-\alpha)$ CS for $\theta_0$, where $\lambda_t$ is a predictable sequence.
\end{lemma}
\begin{proof}
    Note that $\pi_t(a \mid X_t, \Omega_{t-1}) \in [\frac{1}{k_t}, 1 - \frac{1}{k_t}] $ and consequently $h_t \in [-k_t, k_t]$. Inspired by the truncation technique used by \cite[Theorem 1]{waudbysmith2022anytimevalid}, we show that $M_T(\theta_0)$ in Equation~\eqref{eq:test_martingale} is a test martingale,
\begin{equation} \label{eq:test_martingale}
    M_T(\theta_0) := \prod_{t=1}^{T} \left( 1 + \lambda_t(\theta_0) (h_t - \theta_0) \right).
\end{equation}
For $M_T(\theta_0)$ to be a test martingale, we must show $M_0(\theta_0) = 1$, $\{M_T(\theta_0)\}_{t=1}^{T}$ is non-negative, and that $\mathbb{E}^{T-1}\left( M_T(\theta_0) \right) =  M_{T-1}(\theta_0).$

$M_T(\theta_0)$ is non-negative if $\left(1 + \lambda_t(\theta_0)(h_t - \theta_0) \right) > 0$ $\forall t \in 1, \dots, T$. \cite{waudbysmith2022estimating} state this condition in their Proposition 3 as requiring $ \lambda_t(\theta_0)\left(h_t - \theta_0 \right)  > -1 $. Consider the case when $(h_t - \theta_0) < 0$. We have that

$$1 + \lambda_t(\theta_0) (h_t - \theta_0) \geq 1 + \lambda_t(\theta_0) (- k_t - \theta_0). $$

In this case, $\lambda_t(\theta_0) \in (-\infty, \frac{1}{k_t + \theta_0})$ will give
$$ 1 + \lambda_t(\theta_0) (-k_t - \theta_0)  > 1 + \frac{-k_t - \theta_0}{k_t + \theta_0} = 0.$$

Next consider when $(h_t - \theta_0) > 0$, then setting $\lambda_t(\theta_0) \in \left(\frac{-1}{k_t + \theta_0}\right)$ guarantees $\lambda_t(\theta_0) (h_t - \theta_0) > -1$. Taking the union of these sets gives $\lambda_t(\theta_0) \in \left(\frac{-1}{k_t - \theta_0}, \frac{1}{k_t + \theta_0} \right)$, and we conclude that $M_T(\theta_0)$ is non-negative.

Next, we check the condition on the conditional expectation,
\begin{align*}
    \mathbb{E}^{T-1}\left( M_T(\theta_0) \right) &= \mathbb{E}^{T-1} \left(M_{T-1}(\theta_0)\times (1 + \lambda_T(\theta_0)(h_T - \theta_0)  \right) \\
    &= M_{T-1}(\theta_0)(1 + \lambda_T(\theta_0)\mathbb{E}^{T-1}(h_T - \theta_0) \\
    &= M_{T-1}(\theta_0) (1 + \lambda_T(\theta_0) \times 0) = M_{T-1}(\theta_0).
\end{align*}

$M_T(\theta_0)$ is therefore a test martingale. By Ville's inequality for non-negative supermartingales,
$$\mathbb{P} \left( \exists T \in \mathbb{N}^+, M_T(\theta_0) \geq \frac{1}{\alpha} \right) \leq \alpha.$$ 

It follows that the set 
$$C_T^{\text{Betting}} := \left\{ \theta^{'} \in [-1,1] : \prod_{t=1}^{T} \left(1 + \lambda_t(\theta^{'})(h_t - \theta^{'} )\right) < \frac{1}{\alpha}  \right\}, $$
forms a $(1-\alpha)$ confidence set.
\end{proof}

\subsection{Hedged-CS}\label{appdx:hedgedcs_subappendix}

Following suggested values from \cite{waudbysmith2022estimating}, we set
\begin{equation}\label{eq:lambda-default}
\lambda_t = \sqrt{\frac{2 \log (2/ \alpha)}{\hat{\sigma}^2_{t-1} t \log(1 + t)}} \wedge c, \textrm{ where }c = 0.5,
\end{equation}
\begin{equation*}
    \hat{\theta}_t = \frac{\frac{1}{2} + \sum_{i=1}^{t-1} h_i}{t},
\end{equation*}
\begin{equation*}
    \hat{\sigma}_t^2 = \frac{\frac{1}{4} + \sum_{i= 1}^{t}(h_i - \hat{\theta})^2 }{t}.
\end{equation*}
We define
 \begin{equation*}
    \mathcal{K}_T^{+}(\theta') := \prod_{t=1}^{T}(1 + \lambda_t(\theta') (h_t - \theta')) ,
\hspace{4mm}
        \mathcal{K}_T^{-}(\theta') := \prod_{t=1}^{T}(1 - \lambda_t(\theta') (h_t - \theta')) ,
    \end{equation*}
    \begin{equation*}
        \mathcal{M}_T(\theta') := m \mathcal{K}_T^{+}(\theta') + (1-m)\mathcal{K}_T^{-}(\theta'),
    \end{equation*}

where $m=0.5$ (in general, $m \in [0,1]$). Letting $\lambda_t(\theta') = \lambda_t$ as defined in Equation~\eqref{eq:lambda-default}, and truncated to fall within $\left(\frac{-1}{k_t - \theta'}, \frac{1}{k_t + \theta'} \right)$, both $\mathcal{K}_T^{+}(\theta')$ and $\mathcal{K}_T^{-}(\theta')$ are test martingales when $\theta' = \theta_0$. It follows that $\mathcal{M}_T(\theta')$ is also a test martingale when $\theta' = \theta_0$ \cite[Theorem 3]{waudbysmith2022estimating}. By Lemma~\ref{lemma:bettingCS}, \begin{equation*}
        C_T^{\text{Hedged}} := \bigcap_{t \leq T} \left\{ \theta^{'} \in [-1,1] : \mathcal{M}_T(\theta') < \frac{1}{\alpha}  \right\},
    \end{equation*}
forms a valid $(1-\alpha)$-CS. We now focus computing $C_T^{\mathrm{Hedged}}$.

If $\lambda_t$ does not depend on $\theta^{'}$ (apart from truncating the domain), \cite{waudbysmith2022estimating} show that, empirically, $C_T^{\text{Hedged}}$ forms an interval at each time $T$. We can then search over a grid of possible values of $\theta^{'}$ $\in [-1,1]$, and set lower and upper bounds as 
$$L_T^{\text{Hedged}} = \sup_{t \in \{1,\dots, T\} }\inf_{T}\left\{ \theta^{'} \in [-1,1] : \mathcal{M}_T(\theta') < \frac{1}{\alpha}  \right\}, $$
$$U_T^{\text{Hedged}} = \inf_{t \in \{1,\dots, T\} }\sup_T \left\{ \theta^{'} \in [-1,1] : \mathcal{M}_T(\theta') < \frac{1}{\alpha}  \right\}.$$

As a result, $[L_T^{\text{Hedged}}, U_T^{\text{Hedged}}]$ forms a $(1-\alpha)$-CS for $\theta_0$.

\section{Proof of Theorem~\ref{theorem:prpl_empbern}}\label{appdx:prpl_proof}

\begin{proof}
Note that $\xi_t - \hat{\xi}_{t-1} > -1$. Given this fact, \citet[Lemma 1]{waudbysmith2022anytimevalid} show that the process
\begin{equation}\label{eq:test-supermartingale}
M_T = \exp \left\{\sum_{t = 1}^T \lambda_t \left(\xi_t - \frac{\theta_0}{k_t + 1} \right) - \sum_{t = 1}^{T} \left(\xi_t - \hat{\xi}_{t-1} \right)^2 \psi_{E}(\lambda_t)    \right\},
\end{equation}
is a test supermartingale. Using Ville's inequality, they invert $M_t$ to form a $(1-\alpha)$-lower CS. We define an $(1-\alpha)$-Upper CS by defining $\xi_t = \frac{- h_t}{k_t + 1}$, and apply a union bound, which gives the result.

\end{proof}

\section{Proof of Theorem~\ref{theorem:asympCS}}\label{appdx:asympcs_proof}

\paragraph{Proof Outline}

$(h_t)_{t=1}^T$ is recognized to be a sequence of random variables with conditional mean $\theta_0$ and conditional variance $\sigma^2$. This allows us to invoke Theorem 2.5 from \cite{waudbysmith2023timeuniform}. In order to do so, we must verify three assumptions. 
\paragraph{Assumption 1 (Cumulative variance diverges almost surely)} 
This assumption is satisfied in Appendix~\ref{appdx:mds_clt_proof} where we establish that the average conditional variance of $z_t$ (which equals the average conditional variance of $h_t$) does not vanish. It follows that their sum diverges. 
\paragraph{Assumption 2 (Lindeberg-type uniform integrability)} 
We mush show that there exists some $0 < \kappa < 1$ such that
$$\sum_{t=1}^{\infty} \frac{\mathbb{E} \left[ (h_t - \theta_0)^2 \mathbbm{1} \left((h_t - \theta_0)^2 > V_{t}^\kappa  \right) \mid \Omega_{t-1}  \right]}{ V_{t}^\kappa} < \infty \text{ almost surely, }$$
where $V_t= \sum_{i=1}^t \sigma_i^2 $.

As is noted in \cite{waudbysmith2023timeuniform}, this equation is satisfied if $1/K \leq \mathbb{E}~|h_t - \theta_0|^q < K$ a.s. for all $t \geq 1$ and for some constant $K>0$. Without loss of generality, assume $q = 2+\delta$. We have that
$$\mathbb{E} | h_t - \theta_0 |^q \leq \mathbb{E}(|h_t|^q) + \mathbb{E}(|\theta_0|^q).$$

Note that $\mathbb{E}(|h_t|^q) \propto \mathbb{E}(|Y_t|^q) < \infty$. Then pick $K^* = K + \mathbb{E}(|h_t|^q) $ and the condition holds.\\
\paragraph{Assumption 3 (Consistent variance estimation)}
We must show that the estimator, $\hat{\sigma}^2_t$, of $\tilde{\sigma}_t^2$ satisfies
$$\frac{\hat{\sigma}^2_t}{\tilde{\sigma}^2_t} \xrightarrow{a.s.} 1.$$

Our estimator is the sample average of the variances estimated thus far. We note that $z_t$ is a square-integrable MDS. Hence, we utilize the Strong Law of Large Numbers for a MDS, and we can establish that the sample average of the squared deviations converges almost surely to the variance of $z_t$. We establish that $\hat{\sigma}^2(z_t) = \hat{\sigma}^2(h_t)$ by showing

\begin{align*}
    \hat{\sigma}^2(z_t) &= \frac{1}{T} \sum_{t= 1}^T \left(z_t - \bar{z_t} \right)^2 \\
    &= \frac{1}{T} \sum_{t= 1}^T \left(h_t-\theta_0 - \frac{1}{T}\sum_{t=1}^T(h_t - \theta_0) \right)^2 \\
    &= \frac{1}{T} \sum_{t=1}^T \left( h_t - \theta_0 + \theta_0 - \bar{h}_T \right)^2 = \hat{\sigma}^2(h_t).
\end{align*}

By the SLLN, $\hat{\sigma}^2(h_t) = \hat{\sigma}^2(z_t) \xrightarrow{a.s.} \text{Var}(z_t) = \text{Var}(h_t)$.

\section{Implementation Details}\label{appdx:implementation}

\subsection{Performance Metrics}

We explicitly define ``Cumulative Error Probability'' and ``Power'', the performance metrics shown in Figures~\ref{fig:miscoverage-bernoulli},~\ref{fig:results_knnbernoulli}, and~\ref{fig:results_knnbounded}. These are both functions of the sample size. 

At time $T$, the Cumulative Error Probability (Error) is defined as $$\mathrm{Error}(T) := \mathbb{P}(\exists t \in \{50,\dots,T\}~\mathrm{ s.t. }~\theta_0 \not \in [L_t,U_t] ).$$
We can empirically estimate this probability over 1000 repetitions of our simulation. For each iteration, we construct confidence sets at each time. We denote the confidence set at time $T$ for iteration $i$ as $[L_{i,T}, U_{i,T}]$. We define our estimate as
$$\widehat{\mathrm{Error}}(T) := \frac{1}{1000} \sum_{i = 1}^{1000}\mathbbm{1}[\exists t \in \{50,\dots,T\}~\mathrm{ s.t. }~\theta_0 \not \in [L_{i,t},U_{i,t}]]. $$

In Figure~\ref{fig:miscoverage-bernoulli}, we denote the null hypothesis as $\theta_{H_0} = 0$. At time $T$, we denote the power as 
$$\mathrm{Power}(T) := \mathbb{P}(\exists t \in \{50,\dots,T\}~\mathrm{ s.t. }~\theta_{H_0} \not \in [L_t,U_t] ).$$
Our empirical estimate of this function is
$$\widehat{\mathrm{Power}}(T) := \frac{1}{1000} \sum_{i = 1}^{1000}\mathbbm{1}[\exists t \in \{50,\dots,T\}~\mathrm{ s.t. }~\theta_{H_0} \not \in [L_{i,t},U_{i,t}]] .$$

\subsection{$\hat{\theta}^{\mathrm{A2IPW}}$ T-Statistic}\label{appdx:tstat}

In Section~\ref{sec:fixed_time_theory}, Theorem~\ref{theorem:MDS-CLT} gives an asymptotic distribution for the $\hat{\theta}^{\mathrm{A2IPW}}$ estimator which depends on $\sigma^2$. In practice, we typically do not have access to $\sigma^2$ and we must estimate this quantity, denoted as $\hat{\sigma}^2$. With $\hat{\sigma}^2 \xrightarrow[]{p}\sigma^2$, we may invoke Slutsky's Theorem, and use $\hat{\sigma}^2$ in place of $\sigma^2$. Similarly to \cite{kato2021}, we call this our t-statistic,
\begin{equation*}
    \frac{\sqrt{T} (\hat{\theta}^{\mathrm{A2IPW}} - \theta_0)}{\hat{\sigma}^2} \xrightarrow[]{d} N(0,1).
\end{equation*}

In Assumption 3 of Appendix~\ref{appdx:asympcs_proof}, we show that our variance estimator converges almost surely, implying convergence in probability. Our asymptotic CI is
\begin{equation*}
    C_T := \bar{h}_t \pm z_{1-\frac{\alpha}{2}}\frac{\hat{\sigma}^2}{\sqrt{T}},
\end{equation*}
where $\hat{\sigma}^2 = \frac{1}{T} \sum_{t=1}^T \left( h_t - \bar{h}_T \right)^2.$
\subsection{Bernoulli Outcome Simulation}\label{appdx:implementation_bernoulli}
In Section~\ref{sec:bernoulli-sim} and Figure~\ref{fig:miscoverage-bernoulli} plots (a) and (b), we simulate $(X_t, A_t, Y_t)_{t=1}^{T = 5000}$, where
$$\mathbf{X}_t \sim N( [\mathbf{0}_3 ] , \mathbf{I}_3 ) , $$
$$\boldsymbol{\beta}^T = \begin{bmatrix}-2,-3,5\end{bmatrix},$$
$$\pi_t =  \left( \frac{\sqrt{\hat{v}_{t-1}(1,\mathbf{X}_t)}}{\sqrt{\hat{v}_{t-1}(1,\mathbf{X}_t)} + \sqrt{\hat{v}_{t-1}(0,\mathbf{X}_t)}} \right), $$
$$k_t =  \frac{k_{t-1}}{.999}, \hspace{1mm} k_1 = 2 \text{ if method not Kato, else } k_t = 5,$$
$$A_t \sim \text{Bernoulli}\left(\left( \pi_t \vee \frac{1}{k_t} \right) \wedge (1-\frac{1}{k_t}) \right),$$

$$p_t = 0.9\times \text{logit}\left( 0.5 +  \mathbf{X}_t \boldsymbol{\beta} \right) + 0.1A_t,$$
$$Y_t \sim \text{Bernoulli} \left( p = p_t \right).$$
With the data generating process above, $\theta_0 = 0.1$. We ran two separate simulations, where one used k-Nearest Neighbors Regressor (kNN) and the other used Random Forest Regressor (RF) for $\hat{f}$ and $\hat{e}$. We employ sample-splitting and cross-fitting in an effort to avoid over fitting. For the first 50 samples, we let $\pi_t(1 \mid X_t, \Omega{t-1}) = 0.5$ while sufficient samples are collected to give reliable regression estimates. For the regression estimates used in $h_t$, we use sample means conditioned on $A_t$ until $t=50$. We ran 1000 iterations using the DGP above, results for the simulation when RF is used are shown in Figure~\ref{fig:miscoverage-bernoulli}. We provide results for the simulation using kNN in Figure~\ref{fig:results_knnbernoulli}. 

\begin{figure}
    \centering
    \includegraphics[width = 0.9\textwidth]{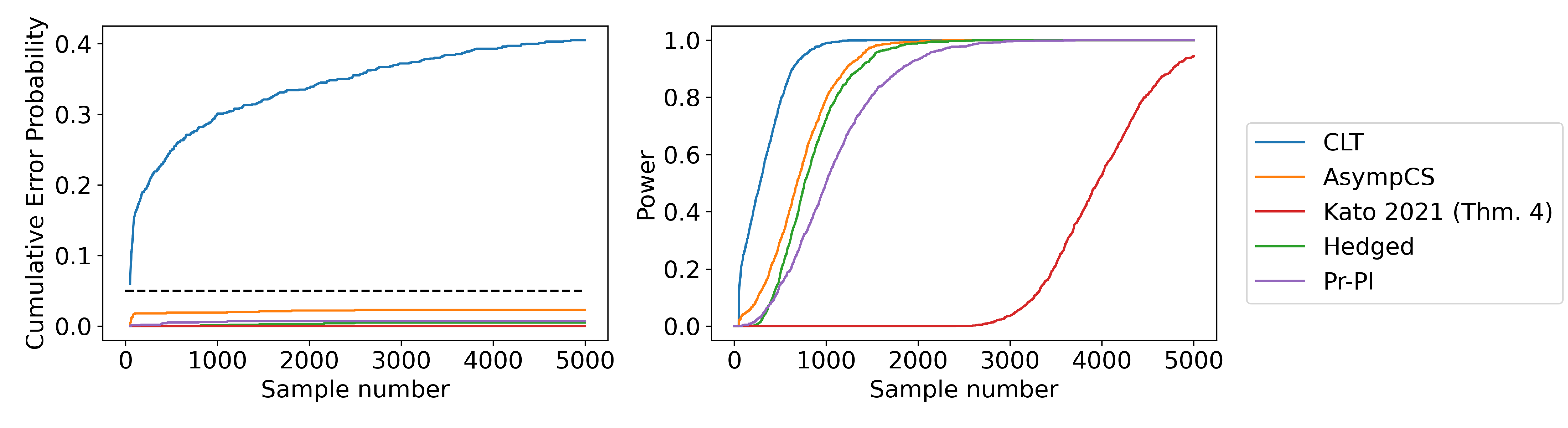}
    \caption{Utilizing a kNN regressor for the protocol used in Figure~\ref{fig:miscoverage-bernoulli}. The policy used for Pr-PI is modified to be truncated within $[0.2,0.8]$.}
    \label{fig:results_knnbernoulli}
\end{figure}

\subsection{Bounded Continuous Outcomes Simulation}\label{appdx:implementation_bounded}

We now consider simulations with a continuous response, as described in Section~\ref{sec:bounded-sim} and shown in Figure~\ref{fig:miscoverage-bernoulli}, plots (c) and (d). Data was simulated as
$$X_i \sim \text{Uniform}(0,1), \hspace{2mm} \text{for } i \in \{1,2,3\},$$
$$\pi_t =  \left( \frac{\sqrt{\hat{v}_{t-1}(1,\mathbf{X}_t)}}{\sqrt{\hat{v}_{t-1}(1,\mathbf{X}_t)} + \sqrt{\hat{v}_{t-1}(0,\mathbf{X}_t)}} \right) ,$$
$$k_t =  \frac{k_{t-1}}{.999}, \hspace{1mm} k_1 = 2 \text{ if method not Kato, else } k_t = 5,$$
$$A_t \sim \text{Bernoulli}\left(\left( \pi_t \vee \frac{1}{k_t} \right) \wedge (1-\frac{1}{k_t}) \right),$$
$$\boldsymbol{\beta}^T = [-0.04, -0.01, 0.05],$$
$$\epsilon_0 \sim \text{Uniform}(-0.05, 0.05,)$$
$$Y_0 = 0.4 + \mathbf{X}  \boldsymbol{\beta}  + \epsilon_0,$$
$$\epsilon_1 \sim \text{Uniform}(-4.5\mathbf{X}  \boldsymbol{\beta}, 4.5\mathbf{X}  \boldsymbol{\beta}),$$
$$Y_1 = 0.4 + \mathbf{X}  \boldsymbol{\beta} + \theta_0 + \epsilon_1.$$

In our simulations we set $\theta_0 = 0.1$. Again, we use kNN and random forest (RF) regressors to estimate $\hat{f}$ and $\hat{e}$. Similarly, we employ sample-splitting and cross-fitting. For the first 50 samples, we again let $\pi_t(1 \mid X_t, \Omega{t-1}) = 0.5$. For the regression estimates used in $h_t$, we use sample means conditioned on $A_t$ until $t=50$. We ran 1000 iterations using the DGP above, results for the simulation when RF is used are shown in Figure~\ref{fig:miscoverage-bernoulli}. We provide results for the simulation using kNN in Figure~\ref{fig:results_knnbounded}. 

\begin{figure}
\includegraphics[width=0.9\textwidth]{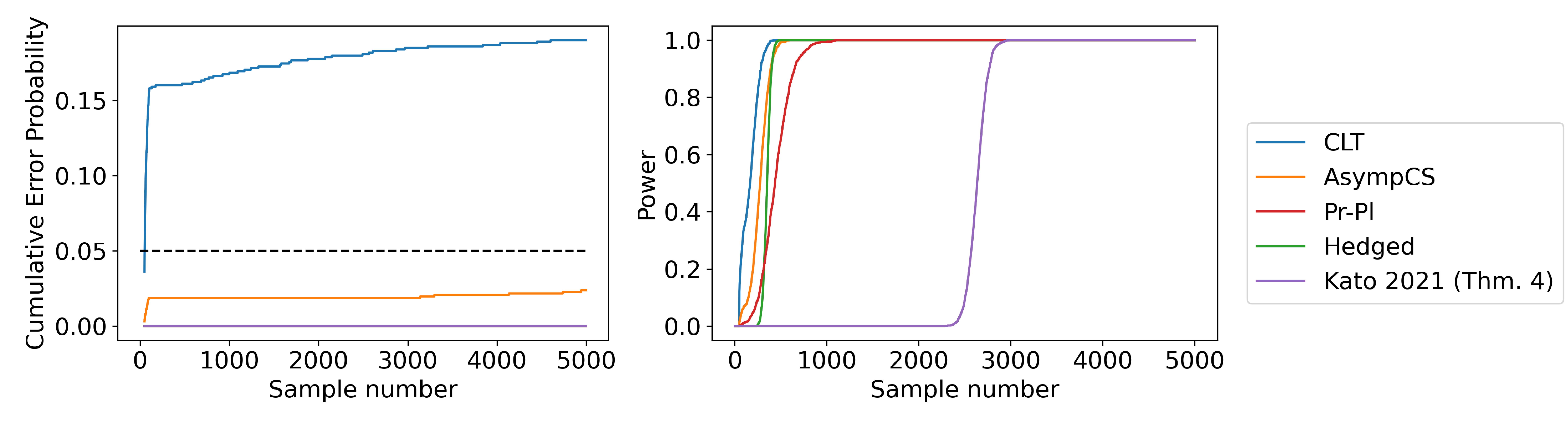}
    \caption{Results for simulation described in Appendix~\ref{appdx:implementation_bounded} using a k-Nearest Neighbor regressor.}\label{fig:results_knnbounded}
\end{figure}
\begin{figure}
    \centering
    \includegraphics[width = 0.5\textwidth]{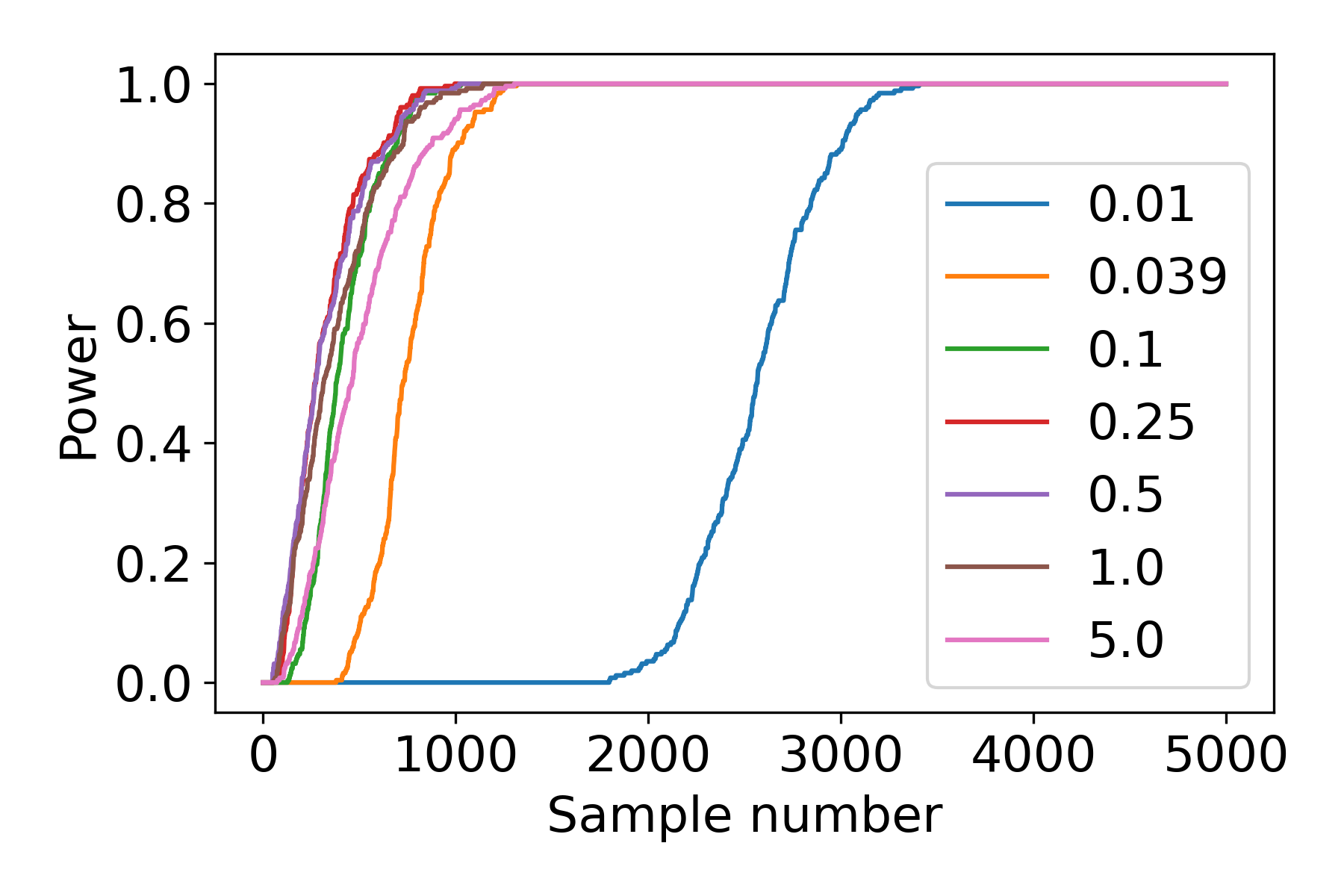} 
    \caption{Power curves for AsympCSs constructed using different values of $\rho$. Curves are based on 256 iterations of the simulation setup described in Appendix~\ref{appdx:implementation_bernoulli}.}
    \label{fig:varyingrho}
\end{figure}

\subsection{Effect of Truncation on Inference}\label{appdx:more_sim}

Although the modifications made to the DGP in Appendix~\ref{appdx:implementation_bernoulli} are only minor made to the simulation described in Section~\ref{sec:truncation-sim} and shown in Figure~\ref{fig:truncate}, we provide an explicit DGP for completeness. We assume that we have oracle access to $v$, allowing us to calculate $\pi^{\mathrm{AIPW}}$. We set $k_t$ to be a constant. The DGP is
$$\mathbf{X}_t \sim N( [\mathbf{0}_3 ] , \mathbf{I}_3 ) , $$
$$\boldsymbol{\beta}^T = \begin{bmatrix}-2,-3,5\end{bmatrix},$$
$$\pi_t =  \left( \frac{\sqrt{{v}(1,\mathbf{X}_t)}}{\sqrt{{v}(1,\mathbf{X}_t)} + \sqrt{{v}(0,\mathbf{X}_t)}} \right), $$
$$k_t =  \frac{1}{\pi_{t,min}},$$
$$A_t \sim \text{Bernoulli}\left(\left( \pi_t \vee \frac{1}{k_t} \right) \wedge (1-\frac{1}{k_t}) \right),$$

$$p_t = 0.1\times \text{logit}\left( 0.5 +  \mathbf{X}_t \boldsymbol{\beta} \right) + 0.4A_t,$$
$$Y_t \sim \text{Bernoulli} \left( p = p_t \right).$$

\subsection{Selecting $\rho$ for an AsympCS}
When constructing an AsympCS, the analyst must select a value for $\rho$. If the analyst wishes to minimize width of the interval produced at a specific sample size, $T$, then the analyst can accomplish this by setting 
$$\rho =\sqrt{\frac{-2 \log\alpha + \log (-2 \log \alpha + 1)}{T}}.$$
In practice, the analyst may not have prior knowledge of the effect size magnitude or may not know how long the experiment could last. In this case, it may not be clear for which $T$ $\rho$ should be tuned to. In our simulations we begin constructing CSs at a sample size of $T = 50$. For the sake of simplicity in presentation, we chose to set $\rho = 0.5$ across all experiments. Setting $\rho=0.5$ yields an AsympCS with tight intervals approximately at the start of inference. To understand the effect of setting $\rho = 0.5$ on the performance of the AsympCS, we performed 256 iterations of the Bernoulli outcome simulation described in Appendix~\ref{appdx:implementation_bernoulli} while varying $\rho$. We found that setting $\rho = 0.5$ for this scenario is a reasonable choice and the resulting AsympCS produces intervals with widths that allow for high power early in the experiment. Figure~\ref{fig:varyingrho} shows power curves of the AsympCSs constructed using different levels of $\rho$. 

\end{document}